\documentclass{article}

\usepackage{PRIMEarxiv}

\usepackage[utf8]{inputenc} % allow utf-8 input
\usepackage[T1]{fontenc}    % use 8-bit T1 fonts
\usepackage{url}
\usepackage{graphicx}
\usepackage{tikz,ifthen}
\usepackage{wrapfig}
\usetikzlibrary{arrows,trees,backgrounds,automata,shapes,decorations,plotmarks,fit,calc,positioning,shadows,chains}
\usepackage{stackengine}
\usepackage{amsmath}
\usepackage{amsthm}
\usepackage{mathrsfs}
\usepackage{amscd}
\usepackage{amsfonts}
\usepackage{tabularray}
\usepackage{multirow}
\usepackage{booktabs}
\usepackage{algorithm}
\usepackage[noend]{algorithmic}
\usepackage[dvipsnames]{xcolor}
\usepackage{wrapfig}
\usepackage{natbib}

\theoremstyle{plain}
\newtheorem{theorem}{Theorem}[section]
\newtheorem{proposition}[theorem]{Proposition}
\newtheorem{lemma}[theorem]{Lemma}

\theoremstyle{definition}
\newtheorem{definition}[theorem]{Definition}

\theoremstyle{remark}

\definecolor{color0}{RGB}{240, 78, 64} %output
\definecolor{color4}{RGB}{60, 220, 125} %input
\definecolor{color6}{RGB}{120, 100, 200} %hidden
\definecolor{color7}{RGB}{107, 100, 200} %hidden

\definecolor{nnedgecolor}{RGB}{90,90,90}
\tikzstyle{every pin edge}=[<-,shorten <=1pt]
\tikzstyle{every path}=[draw=color7!50]
\tikzstyle{neuron}=[circle,fill=black!25,minimum size=17pt,inner sep=0pt]
\tikzstyle{input neuron}=[neuron, fill=color4]
\tikzstyle{output neuron}=[neuron, fill=color0]
\tikzstyle{hidden neuron}=[neuron, fill=color6]
\tikzstyle{annot} = [text width=4em, text centered]
\tikzstyle{nnedge} = [-{stealth},shorten >=0.1cm, shorten <=0.05cm,line 
width=0.8pt,nnedgecolor]
\tikzstyle{nnedge_t} = [-{dashed},shorten >=0.1cm, shorten <=0.05cm,line 
width=0.8pt,nnedgecolor]
\usetikzlibrary{calc}

\usepackage{amsmath,amsfonts,bm}

\def\eqref#1{equation~\ref{#1}}
\def\1{\bm{1}}

\def\va{{\bm{a}}}
\def\vb{{\bm{b}}}
\def\vc{{\bm{c}}}

\def\vl{{\bm{l}}}

\def\vu{{\bm{u}}}
\def\vv{{\bm{v}}}
\def\vw{{\bm{w}}}
\def\vx{{\bm{x}}}

\def\vz{{\bm{z}}}

\def\mA{{\bm{A}}}

\def\mD{{\bm{D}}}

\def\mW{{\bm{W}}}

\DeclareMathAlphabet{\mathsfit}{\encodingdefault}{\sfdefault}{m}{sl}
\SetMathAlphabet{\mathsfit}{bold}{\encodingdefault}{\sfdefault}{bx}{n}

\newcommand{\E}{\mathbb{E}}

\usepackage{hyperref}

\title{\large Probabilistically Tightened Linear Relaxation-based Perturbation Analysis for Neural Network Verification\thanks{Accepted at the Journal of Artificial Intelligence Research (JAIR). DOI: \url{https://doi.org/10.1613/jair.1.20808}}}

\author{
  Luca Marzari, Ferdinando Cicalese and Alessandro Farinelli
  \\
  Department of Computer Science, University of Verona, Verona, Italy.\\
  Contact author: \textit{luca.marzari@univr.it}}

\begin{document}
\maketitle

\begin{abstract}
   We present \textbf{P}robabilistically \textbf{T}ightened \textbf{Li}near \textbf{R}elaxation-based \textbf{P}erturbation \textbf{A}nalysis (\texttt{PT-LiRPA}), a novel framework that combines over-approximation techniques from LiRPA-based approaches with a sampling-based method to compute tight intermediate reachable sets. In detail, we show that with negligible computational overhead, \texttt{PT-LiRPA} exploiting the estimated reachable sets, significantly tightens the lower and upper linear bounds of a neural network's output,  reducing the computational cost of formal verification tools while providing probabilistic guarantees on verification soundness. Extensive experiments on standard formal verification benchmarks, including the International Verification of Neural Networks Competition, show that our \texttt{PT-LiRPA}-based verifier improves robustness certificates, i.e., the certified lower bound of $\varepsilon$ perturbation tolerated by the models, by up to 3.31X and 2.26X compared to related work. Importantly, our probabilistic approach results in a valuable solution for challenging competition entries where state-of-the-art formal verification methods fail, allowing us to provide answers with high confidence (i.e., at least 99\%).
\end{abstract}

% keywords can be removed
% \keywords{Control Barrier Functions \and Probabilistic Verification \and Safe Deep Reinforcement Learning}

\section{Introduction}

Deep neural networks (DNNs) and recently large language models (LLMs) \citep{vaswani2017attention} have revolutionized various fields, from healthcare and finance to natural language processing, enabling remarkable capabilities, for instance, in image recognition \citep{image}, robotic tasks such as manipulation \citep{manipulation1,manipulation2} and navigation \cite{curriculum,navigation,TIST, eps_retrain}. Nonetheless, their opacity and vulnerability to the so-called ``adversarial inputs" \citep{adversarial,amir2022verifying} have also raised significant concerns, especially when deployed in safety-critical applications such as autonomous driving or medical diagnosis. Consequently, developing methods to certify the safety aspect of these models is crucial. Achieving provable safety guarantees involves employing formal verification (FV) of neural network techniques \citep{LiuSurvey}, which mathematically ensures that a system will never produce undesired outcomes in all configurations tested.\\
In this work, we focus on robustness verification, which aims to guarantee that a model's output remains consistently robust under small predefined input perturbations around a given point $\vx_0,$ typically defined as an $\ell_p$ ball of $\varepsilon$ radius around the input $\vx_0,$ i.e., the set  $\mathcal{C} = \mathcal{C}_{\vx_0, \varepsilon}=\{\vx \vert \;\vert\vert \vx-\vx_0\vert\vert_p \leq \varepsilon\}.$ A standard approach to proving the robustness property is to first encode the desired output of the classifier as the sign of the value $f$ computed by a single (output) node, i.e., so that the correct classification in $\vx_0$ corresponds to $f(\vx_0) > 0.$  Then, the robustness verification becomes equivalent to checking that  $\min\limits_{\vx\in\mathcal{C}} f(\vx)$, is positive \citep{bcrown}.
%\footnote{For the sake of clarity and without loss of generality, we are only going to discuss the optimal lower bounding $\underline{f}(\vx)$. Similar considerations can also be applied when computing the upper bound $\overline{f}(\vx)$ with the necessary changes in computation.} 
However, due to the non-linear and non-convex nature of the DNN, solving this problem is, in general,  NP-hard \citep{Reluplex,fastlin}. A recent line of works called linear relaxation-based perturbation analysis (LiRPA) algorithms \citep{crown,acrown,bcrown,autolirpa} proposes a formal analysis based on a sound linear relaxation of the DNN. The idea is to compute relaxations of all {\em sources of non-linearity} in the network so as to obtain two linear functions providing respectively a lower and an upper bound of the DNN's output, which are then used in the robustness certification. Nonetheless, these methods use overapproximation techniques to satisfy the worst-case setting, where no exceptions are allowed. While this strict approach ensures absolute safety within the certified input space, it faces crucial limitations. In fact, the efficiency of LiRPA approaches, and, in general, any FV methods, scales poorly with the size of the model. Additionally, formal methods fail to offer meaningful robustness information when input perturbations exceed the certified bounds. As discussed in \citep{proven}, relying solely on formal verification can lead to overly conservative outcomes, particularly when adversarial examples are rare or when prior knowledge of the input distribution is available. This stems from the exact nature of formal solvers, which treat all violations equally: even a single adversarial input—such as a one-pixel change—invalidates the entire region, without distinguishing between isolated, low-probability cases and large, semantically significant unsafe regions.
Similar to recent approaches \citep{proven,randomizedSmoothing}, to overcome these critical limitations, we propose a probabilistic perspective that allows for an infinitesimal trade-off in certainty while providing the likelihood of such violations, enabling more nuanced and practical assessments of robustness that better reflect real-world risks. Importantly, the goal of this work is not to compare probabilistic methods with provable ones, as they are fundamentally different in nature. Rather, our objective is to propose a novel, complementary solution that enhances formal methods by offering additional safety information about the models under evaluation for particular challenging instances to be verified.\\
Building on this foundation, we explore \citet{acrown} speculation that estimates as tightly as possible reachable sets, i.e., the range of possible values at each hidden node given bounded inputs, could significantly enhance verification efficiency by yielding sharper final linear bounds.
In this perspective, we address two key research questions: (i) \textit{How can we compute (probabilistic) reachable sets that yield tighter intermediate and output bounds than existing methods, while remaining computationally efficient?}
(ii) \textit{What theoretical guarantees can we establish for linearized layers using (probabilistically sound) reachable set?} 

\textbf{Our Contributions.} We propose \textbf{P}robabilistically \textbf{T}ightened \textbf{Li}near \textbf{R}elaxation-based \textbf{P}erturbation \textbf{A}nalysis (\texttt{PT-LiRPA}), a novel framework for computing tight linear lower and upper bounds combining existing LiRPA methods with a sampling-based reachable set estimation strategy. Unlike other related probabilistic works  \citep{proven,randomizedSmoothing}, our approach does not rely on specific input distributions or attacks. 
Specifically, in this work, we start considering statistical results on tolerance limits \citep{wilks} to provide probabilistic guarantees on estimating the minimum and maximum values of a neural network's output. Using the approach of \citet{wilks}, we are guaranteed that, for any $R, \psi \in (0,1),$ with probability $\psi$, at most, a fraction $(1-R)$ of points from a possibly infinitely large sample in the perturbation region $\mathcal{C}$ may violate the estimated bounds obtained from an initial sample, whose size $n$ is explicitly computable from the desired parameters $\psi$ and $R$. Hence, our intuition is to extend this approach to also compute an estimation of reachable sets with a specified confidence level. Nonetheless, while the result of \citep{wilks} quantitatively bounds the error of the sample-based procedure by predicting the fraction of potential violations in future samples, it results in "weaker" guarantees with respect to other known probabilistic approaches. In particular, this approach does not provide information on the magnitude of these violations with respect to the estimated bounds. Specifically, in any estimated reachable set for a possibly existing fraction $(1-R)$ of points drawn from $\mathcal{C}$, the probability of violating the bound might in principle be uncontrollably large. 
%and holds only for a fixed fraction $R$ of the perturbation region under consideration. 
To address such an issue, following results on \textit{extreme value theory} (EVT) \citep{haan2006extreme,de1981estimation}, we extend \citet{wilks}' probabilistic guarantees and present two novel bounds on the magnitude of potential errors between the true minimum (Theorem \ref{theorem:worst_case_bound}) and the sample-based estimated one (Theorem \ref{theorem:pt_lirpa_guarantees}). Notably, our final result proves that with negligible computational overhead, each node's reachable set computed using random samples represents with high probability the actual domain of that node for any $x \in \mathcal{C}$. Hence, we show that by employing this sampling-based procedure to compute probabilistically tight reachable set bounds in the neural network and integrating these into the linearization used by LiRPA-based formal verification methods, we are able to obtain significantly tighter lower and upper linear bounds of a neural network’s output, while preserving the verification soundness for any input in the perturbation region and specified confidence level. 

To assess the benefit and effectiveness of our novel framework, we perform an extensive empirical evaluation. We first compare our approach on neural networks trained on MNIST and CIFAR datasets with PROVEN \citep{proven} and Randomized Smoothing \citep{randomizedSmoothing}, the most closely related probabilistic verification approaches. In addition, as a ground truth, we also consider a set of state-of-the-art worst-case robustness verification approaches, namely CROWN\citep{crown}, $\alpha$-CROWN \citep{acrown}, $\beta$-CROWN \citep{bcrown}, and GCP-CROWN \citep{GCP-CROWN}. This first set of experiments shows that with a very high confidence (i.e., $\geq 99\%$), \texttt{PT-LiRPA} improves the certified lower bound of $\varepsilon$ perturbation tolerated by the models up to 3.31x and 2.26x compared with both the corresponding probabilistic approaches of \citep{proven,randomizedSmoothing} and up to 3.62x w.r.t. the worst-case analysis results. Finally, in the second set of experiments, we demonstrate that for challenging instances from the International Verification of Neural Networks Competition (VNN-COMP) \citep{VNN-comp2022,VNN-comp2023}, where state-of-the-art formal verification methods fail to produce a conclusive result, our \texttt{PT-LiRPA}, with a quantifiable level of confidence represents a valuable resource in providing safety information.  \\
The paper is structured as follows. Section~\ref{sec:preliminaries} introduces the necessary background on robustness verification and explains how to perform linear relaxation-based perturbation analysis in a Rectified Linear Unit (ReLU)-based deep neural network. We also include a detailed discussion comparing our probabilistic guarantees with those of existing approaches.  We present the theoretical foundations of \texttt{PT-LiRPA}, along with a running example to illustrate our method for linearizing complex deep neural networks and practical implementation in Section~\ref{sec:method}, and \ref{sec:pt_lirpa_algo}, respectively. Finally, Section~\ref{sec:empirical} reports on the extensive empirical evaluation of our approach on standard benchmarks in DNN verification.

\section{Related Work and Preliminaries}\label{sec:preliminaries}

\subsection{Related Work}\label{sec:related}
\paragraph{Formal Verification.} In recent years, significant research has been dedicated to formal verification and especially to robustness verification \citep{LiuSurvey,wei2024modelverification}. For example, \citep{jair1} shares with our work the goal of achieving certified robustness for neural networks, but differs significantly in both focus and methodology. Their approach is specifically designed for the time-series domain, leveraging statistical constraints and polynomial transformations to generate adversarial examples and derive robustness guarantees. In contrast, our work addresses general classification tasks and introduces a probabilistically sound verification method based on tight linear relaxation bounds. Other related work that integrates recurrent neural network(RNN)-based policy learning with formal verification is presented in \citep{jair2}. Specifically, the authors target policy verification under temporal logic constraints in partially observable setups by extracting finite-state controllers from RNNs to enable model checking. In contrast, our work focuses on classification tasks and specifically robustness verification of general feedforward deep neural networks under input perturbations. For this type of verification is important to cite sound and complete verifiers such as mixed integer programming (MIP) \cite{MIP} and satisfiability module theory(SMT)-based solvers \cite{Reluplex,marabou2}.
The most closely related works to our proposal comprise LiRPA-based verification approaches that focus on increasing the quality of linear bounds of the most popular activation functions, such as ReLU, and more general activation functions. More specifically, \citep{autolirpa} proposes a framework for deriving and computing near-optimal sound bounds with linear relaxation-based perturbation analysis for neural networks. This framework is the base of all the most famous state-of-the-art formal verification tools such as CROWN \citep{crown}, $\alpha$-CROWN \citep{acrown}, $\beta$-CROWN \citep{bcrown}, GCP-CROWN \citep{GCP-CROWN}, the top performer on last years VNN-COMP \citep{VNN-comp2022,VNN-comp2023}. Notably, this approach has also been recently employed to both provably \citep{kotha2023provably} and approximate \citep{zhang2024provable} bounding neural network preimages. 

All these approaches will be used as a worst-case verification result in our experiments. 
\vspace{-2mm}
\paragraph{Sampling-based Approaches for Provable Verification Certificates.} Different approaches have tried to incorporate a sampling-based approach to enhance either the linear relaxation of arbitrary non-linear functions \citep{linsyn,sol,dualapp} or the verification process \citep{geometric_sample}, but still maintaining provable verification certificates. For instance, \citep{sol,linsyn} proposed a method synthesizing linear bounds for arbitrary complex activation functions, such as GeLU \citep{gelu} and Swish \citep{swish}, by combining a sampling technique with an LP solver to synthesize candidate lower and upper bound coefficients and then certifying the final result via SMT solvers \citep{smtsolver}. 
In \citep{dualapp}, the authors focus on estimating the actual domain of an activation function by combining Monte Carlo simulation and gradient descent methods to compute an underestimated domain, which is then paired with over-approximations to define provable linear bounds. If, on the one hand, our \texttt{PT-LiRPA} also employs a similar sampling-based procedure to compute the estimated domain of reachable sets, on the other hand, it differs fundamentally in both nature and focus. Specifically, our approach provides an explicit formula that, for any desired level of confidence, gives the number of samples sufficient to estimate the minimum and maximum pre-activation value of any node (over the set of inputs in $\mathcal{C}$) with the given confidence. In addition, we are also able to estimate the maximum error between our estimates and the true extremal pre-activation values. Crucially, our approach adopts a probabilistic perspective to derive safety insights, whereas \citep{dualapp} relies on combining under- and over-approximations to produce provable bounds—an approach that may inherit the limitations of formal methods discussed earlier.

\vspace{-2mm}
\paragraph{Probabilistic Verification.} Recently, several works have explored probabilistic verification of machine learning models. For example, \citep{probabilistic1,probabilistic2,probabilistic9} focus on a probabilistic verification perspective for general machine learning models (e.g., deep generative/diffusion models), leveraging either uncertainty sources 
%encoded from distributions 
generated by specific encoders or using symbolic reasoning and probabilistic inference. \citep{probabilistic3} proposes a Monte Carlo-based method utilizing multi-level splitting to estimate the probability of rare events for robustness verification. Other approaches rely on Chernoff-Cramér bounds, e.g., \citep{probabilistic8} estimates the local variation of neural network mappings at training points to regularize the loss function.  In \citep{probabilistic10}, the authors probabilistically certify a neural network by overapproximating input regions where robustness is violated.
%, but they do not consider general semantic perturbations such as image rotation or translation. To address this limitation, 
In \citep{probabilistic4} a probabilistic certification approach is proposed that can be used in general attack settings to estimate the probability of a model failing if the attack is sampled from a certain distribution.  Another line of work, like \citep{CLEVER} uses \textit{extreme value theory} (EVT) \citep{haan2006extreme} to statistically estimate the local Lipschitz constant and assess the probabilistic robustness of the model without relying on the overapproximation of LiRPA-based solvers. Finally, recent works \citep{probabilistic5,marzari2023dnn,marzari2024enumerating,probabilistic6,probabilistic7} focus either on different types of verification, such as probabilistic enumeration of (un)safe region for neural networks, or on probabilistic specifications for robustness verification, which falls outside the scope of this paper, as we are interested in standard robustness verification with common specifications.

More closely related to our work are the approaches used in PROVEN \citep{proven} and Randomized Smoothing \citep{randomizedSmoothing}, which also focus on probabilistic robustness guarantees for neural network classifiers. PROVEN builds upon LiRPA-based techniques, combining them with concentration inequalities to derive probabilistic bounds on the network’s output under input perturbations. Similarly, Randomized Smoothing constructs a smoothed classifier by adding random noise, typically Gaussian, to the input and then provides robustness guarantees by analyzing the output distribution of the smoothed model. \texttt{PT-LiRPA} aligns with these approaches in that it also leverages linear relaxation techniques and probabilistic reasoning, but introduces a novel perspective by focusing on how much sampling-based underestimations affect the intermediate linear bounds used during network linearization. This enables a finer-grained analysis of robustness with strong probabilistic guarantees, which is particularly useful in scenarios where the available methods for computing exact worst-case bounds either fail (by not terminating under reasonable time and space constraints) or terminate with only over-conservative bounds.
%probabilistic analysis of robustness, particularly useful in scenarios where exact worst-case bounds are overly conservative or intractable to compute.

\subsection{Notation and Problem Formulation}\label{sec:problem_form}
Consider a neural network classifier $f: \mathbb{R}^{d} \to \mathbb{R}$, where $d$ is the input space dimension. Let $N$ denote the number of layers. For each $i = 1, \dots, N,$ we let $d_i$ be the number of nodes in layer $i.$ We use $z^{(i)}_j$ to denote the $j$th node in layer $i$ (according to some fixed ordering of the nodes in the same level). For a given input vector $\vx,$ we 
associate to node $z^{(i)}_j$ two values: the 
preactivation value, denoted by $z^{(i)}_j(\vx)$, 
and the postactivation value $\hat{z}^{(i)}_j(\vx)$ obtained by applying a (typically non-convex) activation function $\sigma$ to the 
preactivation value, i.e.,
$\hat{z}^{(i)}_j(\vx) = \sigma(z^{(i)}_j(\vx)).$
The preactivation value of node $z^{(i)}_j$ is 
obtained as a linear combination of the post-activation values of the nodes in the previous layer. In formulas, let
$\vz^{(i)}(\vx) = (z^{(i)}_1(\vx), \dots, z^{(i)}_{d_i}(\vx))$ 
and 
$\hat{\vz}^{(i)}(\vx) = (\hat{z}^{(i)}_1(\vx), \dots, \hat{z}^{(i)}_{d_i}(\vx) = 
\sigma(\vz^{(i)}(\vx)) = (\sigma(z^{(i)}_{1}(\vx)), \dots, \sigma(z^{(i)}_{d_i}(\vx))).$ 
Then, $\vz^{(i)}(\vx) = \mW^{(i)}\hat{\vz}^{(i-1)}(\vx) + \vb^{(i)},$ for some given inter level weight matrix $\mW^{(i)} \in \mathbb{R}^{d_i\times d_{i-1}}$ and bias vector $\vb^{(i)}\in \mathbb{R}^{d_i}$---as resulting from the network training.  

As observed in the introduction, we assume, without loss of generality, that there is a single node in 
the $N$th layer, which we simply denote by $z^{(N)}$.\footnote{
One can simply enforce this condition for networks that do not satisfy this assumption by adding one layer and encoding, for instance, the robustness property that we aim to verify in a single output node as a margin between logits, which produces a positive output only if the correct label is predicted \citep{LiuSurvey,bcrown}.} Hence we have $f(\vx) = \hat{\vz}^{(N)}(\vx) = {z}^{(N)}(\vx).$

In our following arguments, we assume that the activation function of each node is the ReLU, which is the most employed in the verification works of the literature \citep{acrown,bcrown}, but the soundness of the proposed approach can also be extended for different non-linear scalar functions studied in the literature, such as Tanh, Sigmoid, GeLU, as in \citep{autolirpa}. Hence, we define the robustness verification problem of deep neural networks as follows.

Given an input point of interest $\vx_0$, for which $f(\vx_0) > 0$, and an input perturbation region $\mathcal{C}= \mathcal{C}_{\vx_0, \epsilon}=\{\vx \vert \;\vert\vert \vx-\vx_0\vert\vert_\infty \leq \epsilon\}$, i.e., we set $p = \infty$, obtaining an N-dimensional hypercube, we aim to find, if there exists, an input $\vx \in \mathcal{C}$ such that $f(\vx) \leq 0$, thus resulting in a violation of the property. If $f(\vx) > 0\; \forall \vx \in \mathcal{C}$, we say $f$ is robust (or verified) to all the possible input perturbations in $\mathcal{C}$. Importantly, as we will show, \texttt{PT-LiRPA} can be seamlessly integrated into any state-of-the-art LiRPA method. Therefore, whenever the underlying method supports the verification of specifications beyond the $\ell_\infty$ norm—such as non-convex specifications or related constraints—\texttt{PT-LiRPA} remains directly applicable. A possible way to prove the property is to solve the optimization problem in terms of $\min\limits_{\vx\in\mathcal{C}} f(\vx)$ and by checking if the result is positive.
Formally:
\begin{definition}[\textit{Robustness verification problem}]
\label{def:decision_problem_new}
\phantom{a}

    {\bf Input}: A tuple $\mathcal{T}=\langle f, \mathcal{C}\rangle$.

%a verification problem returns 
    {\bf Output}: $\texttt{Robust} \iff$
    %$ \Leftrightarrow 
    $\min\limits_{\vx\in\mathcal{C}} f(\vx) := z^{(N)}(\vx) > 0$.
\end{definition}

Because of the effect of the activation functions $\sigma$ applied to the value computed in each node, the resulting function $f$ computed by the network is, in general, non-convex, making the above Robustness verification problem NP-hard
\citep{Reluplex}. In order to cope with this issue, (in)complete verifiers usually relax the DNNs' non-convexity to obtain over-approximate sound lower and upper bounding functions on $f$, respectively, denoted by $\underline{f}$ and $\overline{f},$ i.e., $\underline{f}(\vx) \leq f(\vx) \leq \overline{f}(\vx)$ for all $\vx \in \mathcal{C}.$ 
Therefore, if $\underline{f}^* = \min_{\vx \in \mathcal{C}}\underline{f}(\vx) > 0$, then also $f^* = \min_{\vx \in \mathcal{C}} f(\vx) > 0$, i.e., the real minimum value of $f$ will be positive, and similarly if $\overline{f}^* = \min_{\vx \in \mathcal{C}}\overline{f}(\vx) \leq 0$ than also $f^* = \min_{\vx \in \mathcal{C}} f(\vx) \leq 0$.

%If $\underline{f}^* = \min_{\vx \in \mathcal{C}}\underline{f}(\vx) > 0$, then also $f^* = \min_{\vx \in \mathcal{C}} f(\vx) > 0$, i.e., the real minimum value of $f$ will be positive, and similarly if $\overline{f}^* = \min_{\vx \in \mathcal{C}}\overline{f}(\vx) \leq 0$ than also $f^* = \min_{\vx \in \mathcal{C}} f(\vx) \leq 0$. 

In both these situations, we can return a provable result. However, if $\underline{f}^* < 0 < \overline{f}^*$, we cannot be sure about the sign of $f^*$, and we typically have to proceed with a branch and bound (BaB) \citep{bab} process. More specifically, many FV tools firstly recursively divide the original verification problem into smaller subdomains, either by dividing the perturbation region \citep{reluval} or by splitting ReLU neurons into positive/negative linear domains \citep{babSplitRelu}. Secondly, they bound each subdomain with specialized (incomplete) verifiers, typically linear programming (LP) solvers \citep{LP1}, which can fully encode neuron split constraints. The verification process ends once we verify all the subdomains of this searching tree, or we find a single counterexample $\vx$ such that  $\overline{f}(\vx) \leq 0$. Even though LP-verifiers are mainly used in complete FV tools, recent LiRPA-based approaches \citep{acrown,bcrown} show how to solve an optimization problem that is equivalent to the costly LP-based methods with neuron split constraints while maintaining the efficiency of bound propagation techniques, significantly outperforming LP-verification time thanks to GPU acceleration.

\subsection{Linear Relaxation-based Perturbation Analysis (LiRPA) Approaches} \label{sec:lirpa_preliminaries}
LiRPA approaches \citep{crown,singh2019abstract,autolirpa,acrown,bcrown} propose to cope with the non-linearity of the function computed by a neural network by computing linear approximations of each non-linear unit.
The high-level idea is to compute bounds on each neuron's function in the DNN, for instance, all the ReLU nodes, that can be expressed by linear functions, for which the above {\em robustness verification problem} can be efficiently solved. 
In detail, using interval bound propagation (IBP) \citep{lomuscio2017approach}, we first compute a 
{\em reachable set} for each neuron $z^{(i)}_j$.

\begin{definition}[Reachable set]
    Given an input perturbation region $\mathcal{C} \subseteq \mathbb{R}^d$, the reachable set of a neuron $z^{(i)}_j$ is defined as the interval
    $[l^{(i)}_j, u^{(i)}_j],$
    where $l^{(i)}_j \leq z^{(i)}_j(\mathbf{x}) \leq u^{(i)}_j \quad \text{for all } \mathbf{x} \in \mathcal{C}$. 
\end{definition}

\begin{wrapfigure}{l}{0.45\linewidth}
    \centering
    \vspace{-2mm}
    \includegraphics[width=\linewidth]{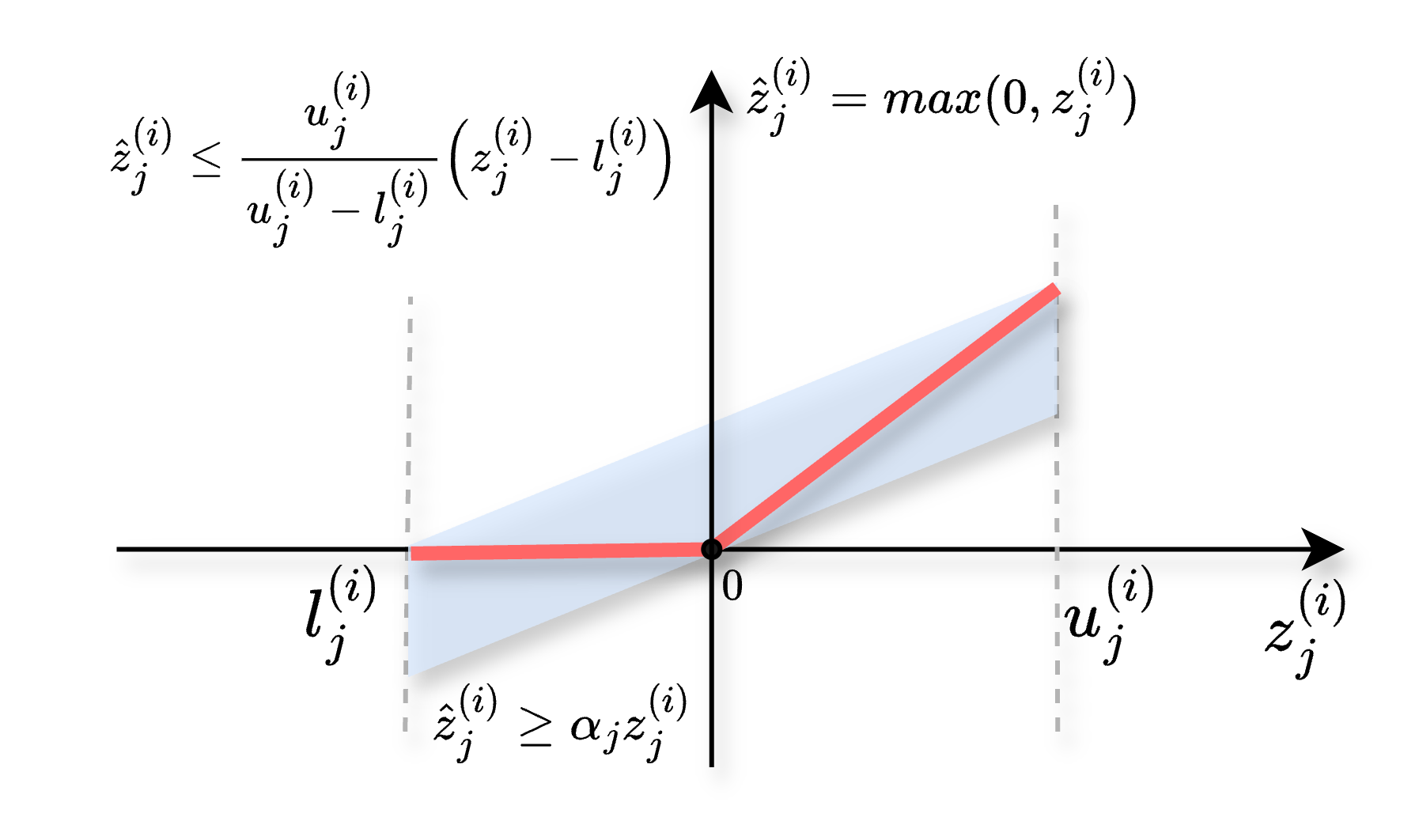}
    \caption{Linear relaxation for ReLU($z^{(i)}_j$)}
    \label{fig:relu}
\end{wrapfigure}
That is, $l^{(i)}_j$ and $u^{(i)}_j$ represent lower and upper bounds, respectively, on the pre-activation values of neuron $z^{(i)}_j$ over the entire perturbation region. The ReLU post-activation value of the node is given by $\hat{z}^{(i)}_j=\max(0, z^{(i)}_j)$.
Therefore, the node is considered \textit{``unstable"} if its pre-activated bounds are such that $u^{(i)}_j > 0 > l^{(i)}_j$ and a linear approximate bound can be computed as depicted in Fig. \ref{fig:relu}.
% \begin{figure}[h!]
%      \centering
%     \includegraphics[width=0.5\linewidth]{images/ReLU.png}
%     \caption{Linear relaxation for ReLU($z^{(i)}_j$)}
%     \label{fig:relu}
% \end{figure}
Once linear bounds are established across all neurons, two propagation methods are typically employed: forward and backward. In forward propagation, the linear bounds for each neuron are expressed in terms of the input and propagated layer by layer until the output is reached. 
In backward propagation, we start from the output and propagate the bounds backward to earlier layers until we can express a linear relation between input and output.
\begin{algorithm}[t]
\caption{\texttt{LiRPA}\cite{crown} backward output bounds computation}\label{alg:lirpa}
\begin{algorithmic}[1]
\small
\STATE \textbf{Input:} A DNN $f$ with $N$ layers, and input $\vx_0$ and an $\varepsilon$ perturbation to compute the perturbation region $\mathcal{C}$, \textit{interm\_bounds} (optional).
\STATE \textbf{Output:} provable lower bound of $f$ when considering $\mathcal{C}$.
\vspace{0.2cm}

\IF{\textit{interm\_bounds} $==\emptyset$}
    \STATE{\textit{interm\_bounds} $\gets \texttt{IBP}(f, \mathcal{C})$}
\ENDIF

\STATE{$\underline{\mA}^{(N)} = I$ ,\; $\underline{\mA}^{(N-1)} = \vw^\top$} \hspace*{\fill} $\rhd$  $\vw^\top$ is the weight vector of the final layer.  
\FOR{$i \in \{N-2,\ldots,1\}$}
    \STATE $\underline{\mD}^{(i)} \gets \texttt{ComputeDiagonalMatrix}$(interm\_bounds$[i]$, $\underline{\mA}^{(i+1)}$) \hspace*{\fill} \hspace*{\fill} $\rhd$  as in Eq. \ref{eq:matrix_bias}
    \STATE $\underline{\vb}^{(i)} \gets \texttt{ComputeBiasVector}$(interm\_bounds$[i]$, $\underline{\mA}^{(i+1)}$)\hspace*{\fill} $\rhd$  as in Eq. \ref{eq:matrix_bias}
    \STATE $\underline{\mA}^{(i)} \gets \underline{\mA}^{(i+1)} \, \underline{\mD}^{(i)} \, \mW^{(i)}$ \hspace*{\fill} $\rhd$  $\mW^{(i)}$ is the weight matrix of layer $i$.
\ENDFOR
\STATE $\underline{d} \gets \underline{\mA}^{(N-1)}\underline{\vb}^{(N-2)}+\ldots+\underline{\mA}^{(2)}\underline{\vb}^{(1)}$
\STATE \textit{lower\_bound} $\gets -\vert\vert \underline{\mA}^{(1)} \vert\vert_1 \cdot \varepsilon + \underline{\mA}^{(1)} \vx_0 + \underline{d}$  \hspace*{\fill} $\rhd$  using Hölder’s inequality for $\min\limits_{\vx \in \mathcal{C}} \underline{\mA}^{(1)} (\vx) + \underline{d}$
\STATE{\textbf{return} \textit{lower\_bound}}
\end{algorithmic}
\end{algorithm}
Otherwise, the node is either considered \textit{``active"} if $l^{(i)}_j \geq 0$ or \textit{``inactive"} if $u^{(i)}_j \leq 0$. 
In detail, to improve the tightness of the bounds, \citep{autolirpa} proposes a refined backward computation strategy that leverages information from a prior forward propagation. We report in Alg. \ref{alg:lirpa} a brief overview of the method. After computing all reachable sets using a method such as interval bound propagation \cite{lomuscio2017approach}, the backward analysis constructs layer-wise linear relaxations, starting from the output layer and working backward to the input. 
For clarity purposes, in the following, we will only report the notation for the linear lower bound computation, but similar considerations also apply to the upper bound. 

For a DNN composed of $N$ layers, the base case is defined as $\underline{\mA}^{(N)} = I$. Moreover, in the case of a single output node, $\underline{\mA}^{(N-1)} = \vw^\top$, where $\vw$ is the weight vector of the final layer. For the remaining layers $i \in \{N-2, \dots, 1\}$, the linear relaxation is propagated using the recurrence $\underline{\mA}^{(i)} = \underline{\mA}^{(i+1)} \, \underline{\mD}^{(i)} \, \mW^{(i)}$, where $\mW^{(i)}$ is the inter level weight matrix and $\underline{\mD}^{(i)}$ is a diagonal matrix encoding the linear relaxation of the activation function at the $i$-th layer. For each layer $i$, we also recursively compute a bias vector $\underline{\vb}^{(i)}$ based on the reachable sets of the $i$-th layer nodes and the matrix $\underline{\mA}^{(i+1)}$. Each diagonal entry $\underline{\mD}^{(i)}_{j,j}$ depends on the preactivation bounds of neuron $j$ (computed during the forward pass) and the sign of $\underline{\mA}_j^{(i+1)}$, i.e., the coefficient associated with neuron $j$ in the linear relaxation of the next layer, for which we report the explicit formulas used for computing $\underline{\mD}^{(i)}$ and $\vb^{(i)}$ in the example provided below.
At the end of the process, a provable lower bound on the minimum of $f$ in $\mathcal{C}$ (the  $\ell_\infty$ norm ball around $\vx_0$) is then easily obtained using Hölder’s inequality \citep{crown} as 
    \begin{equation}
        \min \limits_{\vx\in \mathcal{C}} \va^T_{\text{LiRPA}}(\vx) + \vc_{\text{LiRPA}} = -\vert\vert \underline{\mA}^{(1)} \vert\vert_1 \cdot \varepsilon + \underline{\mA}^{(1)} \vx + \vc_{\text{LiRPA}}.
    \end{equation}
with $\vc_{\text{LiRPA}} = \underline{\mA}^{(N)}\underline{\vb}^{(N-1)}+\ldots+\underline{\mA}^{(2)}\underline{\vb}^{(1)}$.
%where $\va^T_{\text{LiRPA}} = \underline{\mA}^{(1)}$ and %$\vc_{\text{LiRPA}}= c$, 
%the coefficients of the linear equation for the lower bound of $f(\vx)$. For clarity purposes, in the following, we provide an explanatory example of linear bounds computation for a toy DNN.

To provide the reader a concrete and practical illustration of this approach, in the following, we present a simple example of linear bound computation for the toy DNN shown in Figure \ref{fig:toyDNN}, using CROWN, a state-of-the-art LiRPA-based method \citep{crown}.

\paragraph{Example of linear computation with LiRPA}\label{example_computation}

\begin{figure}[b]
    \centering
    \def\layersep{2.0cm}
\begin{tikzpicture}[shorten >=1pt,->,draw=black!50, node
				distance=\layersep,font=\footnotesize]
				
				\node[input neuron] (I-1) at (0,-1) {$x_1$};
				\node[input neuron] (I-2) at (0,-2.5) {$x_2$};

                    % vector input
				\node[left=-0.05cm of I-1] (b1) {};
				\node[left=-0.05cm of I-2] (b2) {};
				
				\node[hidden neuron] (H-1) at (\layersep,-1) {$z_1^{(1)}$};
				\node[hidden neuron] (H-2) at (\layersep,-2.5) {$z_2^{(1)}$};
                   				
				\node[hidden neuron] (H-4) at (2*\layersep,-1) {$\hat{z}_1^{(1)}$};
				\node[hidden neuron] (H-5) at (2*\layersep,-2.5) {$\hat{z}_2^{(1)}$};

                    \node[hidden neuron] (H-7) at (3*\layersep,-1) {$z_1^{(2)}$};
				\node[hidden neuron] (H-8) at (3*\layersep,-2.5) {$z_2^{(2)}$};
                    		
				\node[hidden neuron] (H-10) at (4*\layersep,-1) {$\hat{z}_1^{(2)}$};
				\node[hidden neuron] (H-11) at (4*\layersep,-2.5) {$\hat{z}_2^{(2)}$};

				\node[output neuron] at (5*\layersep, -1.75) (O-1) {$f(x)$};
				
				% Connect every node in the hidden layer with the output layer
				\draw[nnedge] (I-1) --node[above] {$2$} (H-1);
				\draw[nnedge] (I-1) --node[above, pos=0.3] {$-3$} (H-2);
                  
				\draw[nnedge] (I-2) --node[below, pos=0.3] {$1$} (H-1);
				\draw[nnedge] (I-2) --node[below] {$4$} (H-2);

				\draw[dashed,->] (H-1) --node[below] {ReLU} (H-4);
				\draw[dashed,->] (H-2) --node[below] {ReLU} (H-5);

				\draw[nnedge] (H-4) --node[above] {$4$} (H-7);
                    \draw[nnedge] (H-4) --node[above, pos=0.3] {$2$} (H-8);
                   
                    \draw[nnedge] (H-5) --node[below, pos=0.4] {$-2$} (H-7);
				\draw[nnedge] (H-5) --node[below, pos=0.3] {$1$} (H-8);

                    \draw[dashed,->] (H-7) --node[below] {ReLU} (H-10);
                    \draw[dashed,->] (H-8) --node[below] {ReLU} (H-11);
                    \draw[nnedge] (H-10) --node[above] {$-2$} (O-1);
				\draw[nnedge] (H-11) --node[below] {$1$} (O-1);

				% result first prop
				\node[below=0.05cm of H-1] (b1) {};
				\node[below=0.05cm of H-2] (b2) {};

                     % Biases
				\node[below=0.05cm of H-4] (b1) {};
				\node[below=0.05cm of H-5] (b2) {};

                      % Biases
				\node[below=0.05cm of H-7] (b1) {};
				\node[below=0.05cm of H-8] (b2) {};
                   
                    % Annotate the layers
		          \node[annot,left of=I-1, node distance=1cm] (hl1) {$[-2, 2]$};
            \node[annot,left of=I-2, node distance=1cm] (hl1) {$[-1, 3]$};

				\node[annot,above of=H-1, node distance=2cm] (hl1) {$\vz^{(1)}$};
            \node[annot,above of=H-1, node distance=1cm] (hl1) {$[-5, 7]$};
            
            \node[annot,below of=H-2, node distance=1cm] {$[-10, 18]$};

             \node[annot,below of=H-5, node distance=1cm] {$[0, 18]$};
    
				\node[annot,above of=H-4, node distance=2cm] (hl2){$\hat{\vz}^{(1)}$};
                \node[annot,above of=H-4, node distance=1cm] (hl2){$[0,7]$};

                    \node[annot,above of=H-7, node distance=2cm] (hl3) {$\vz^{(2)}$ };
                    \node[annot,above of=H-7, node distance=1cm] (hl3) {$[-36, 28]$};
                    \node[annot,below of=H-8, node distance=1cm] (hl3) {$[0, 32]$};
                     \node[annot,above of=H-10, node distance=2cm] (hl4) {$\hat{\vz}^{(2)}$ };

                     \node[annot,above of=H-10, node distance=1cm] (hl4) {$[0, 28]$};

                     \node[annot,below of=H-11, node distance=1cm] (hl4) {$[0, 32]$};

                    \node[annot,right of=O-1, node distance=1.5cm] (o1) {{\color{Green}$\bm{[-33, 18.9]}$}\\$[-56, 32]$\\{\color{purple}$[-42, 24.3]$}};
                
                    \node[annot,above of=O-1, node distance=2.75cm] (o1) {$z^{(3)}$};

			\end{tikzpicture}
    \caption{Toy DNN used in this example. Intervals reported in green are the exact output reachable set computed via MIP, in black are the ones of IBP, and finally, in purple, the results for CROWN considering the input $[[-2,2], [-1,3]]$.}
    \label{fig:toyDNN}
\end{figure}
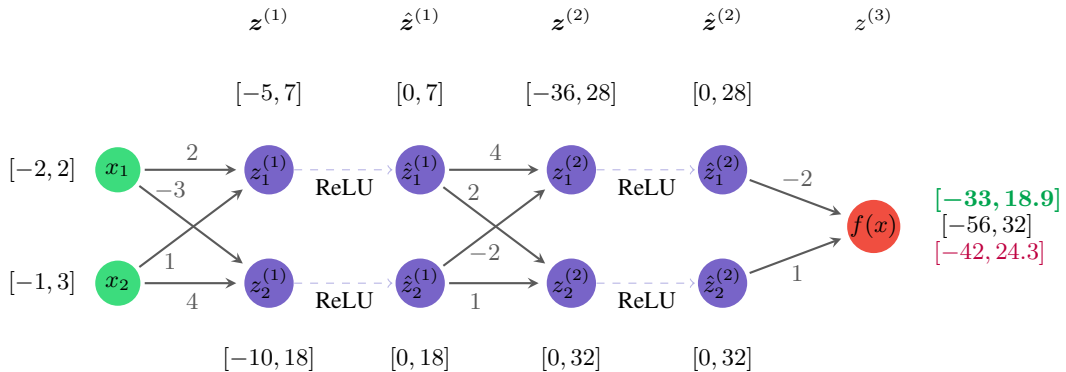

Consider a neural network with two inputs, two hidden layers with ReLU activation, and one single output. Following the notation introduced in  Sec. \ref{sec:preliminaries} we define:
\[
\mW^{(1)}=\begin{bmatrix}
    2       &  1 \\
    -3 & 4
\end{bmatrix},
\quad
\mW^{(2)}=\begin{bmatrix}
4       &  -2 \\
    2 & 1
\end{bmatrix},
\quad 
{\vw^{(3)}}^T = [-2, 1];
\]
and, for simplicity, we set the bias terms in the layers to zero. We consider an original input ${\vx_0}^T= [0, 1]$ and an $\ell_\infty$ $\varepsilon=2$ perturbation around it, thus obtaining a perturbation region $\mathcal{C}=[[-2,2], [-1,3]]$.

By propagating these intervals through the DNN, we obtain the interval $[-56, 32]$ as the output reachable set. Given the reasonable size of the neural network, before computing the linear lower and upper bounds using \texttt{LiRPA}, we employed the sound and complete MIP \citep{MIP} solver to compute the true min and max of the function, respectively, which correspond to $[-33, 18.89]$, highlighted in green in Fig. \ref{fig:toyDNN}.

To compute the lower and upper bounds using CROWN \citep{crown}, we employ LiRPA's backward computation strategy. To this end, it is helpful to represent the neural network as reported in Fig. \ref{fig:alternativeDNN}.

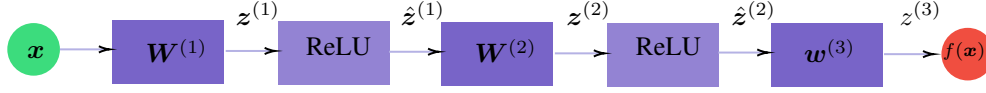
\begin{figure}[h!]
    \centering
    \tikzset{every picture/.style={line width=0.75pt}} %set default line width to 0.75pt        

\begin{tikzpicture}[x=0.6pt,y=0.6pt,yscale=-1,xscale=1]

%Shape: Circle [id:dp5507164524201971] 
\draw  [draw opacity=0][fill={rgb, 255:red, 60; green, 220; blue, 125 }  ,fill opacity=1 ] (30,98.5) .. controls (30,89.39) and (37.39,82) .. (46.5,82) .. controls (55.61,82) and (63,89.39) .. (63,98.5) .. controls (63,107.61) and (55.61,115) .. (46.5,115) .. controls (37.39,115) and (30,107.61) .. (30,98.5) -- cycle ;
%Shape: Rectangle [id:dp3420903647943586] 
\draw  [draw opacity=0][fill={rgb, 255:red, 120; green, 100; blue, 200 }  ,fill opacity=1 ] (96,80) -- (166,80) -- (166,120) -- (96,120) -- cycle ;
%Straight Lines [id:da5834063783763016] 
\draw    (63,98.5) -- (92,98.5) ;
\draw [shift={(94,98.5)}, rotate = 180] [color={rgb, 255:red, 0; green, 0; blue, 0 }  ][line width=0.75]    (6.56,-1.97) .. controls (4.17,-0.84) and (1.99,-0.18) .. (0,0) .. controls (1.99,0.18) and (4.17,0.84) .. (6.56,1.97)   ;
%Shape: Rectangle [id:dp3379583664924073] 
\draw  [draw opacity=0][fill={rgb, 255:red, 120; green, 100; blue, 200 }  ,fill opacity=0.8 ] (200,80.5) -- (270,80.5) -- (270,120.5) -- (200,120.5) -- cycle ;
%Straight Lines [id:da6791400764853078] 
\draw    (167,99.5) -- (196,99.5) ;
\draw [shift={(198,99.5)}, rotate = 180] [color={rgb, 255:red, 0; green, 0; blue, 0 }  ][line width=0.75]    (6.56,-1.97) .. controls (4.17,-0.84) and (1.99,-0.18) .. (0,0) .. controls (1.99,0.18) and (4.17,0.84) .. (6.56,1.97)   ;
%Shape: Rectangle [id:dp7304389114843151] 
\draw  [draw opacity=0][fill={rgb, 255:red, 120; green, 100; blue, 200 }  ,fill opacity=1 ] (302.5,81.5) -- (372.5,81.5) -- (372.5,121.5) -- (302.5,121.5) -- cycle ;
%Straight Lines [id:da5150186200062842] 
\draw    (269.5,100) -- (298.5,100) ;
\draw [shift={(300.5,100)}, rotate = 180] [color={rgb, 255:red, 0; green, 0; blue, 0 }  ][line width=0.75]    (6.56,-1.97) .. controls (4.17,-0.84) and (1.99,-0.18) .. (0,0) .. controls (1.99,0.18) and (4.17,0.84) .. (6.56,1.97)   ;
%Shape: Rectangle [id:dp9386756149865806] 
\draw  [draw opacity=0][fill={rgb, 255:red, 120; green, 100; blue, 200 }  ,fill opacity=0.8 ] (406.5,82) -- (476.5,82) -- (476.5,122) -- (406.5,122) -- cycle ;
%Straight Lines [id:da8689852350305439] 
\draw    (373.5,101) -- (402.5,101) ;
\draw [shift={(404.5,101)}, rotate = 180] [color={rgb, 255:red, 0; green, 0; blue, 0 }  ][line width=0.75]    (6.56,-1.97) .. controls (4.17,-0.84) and (1.99,-0.18) .. (0,0) .. controls (1.99,0.18) and (4.17,0.84) .. (6.56,1.97)   ;
%Shape: Rectangle [id:dp1515485615765505] 
\draw  [draw opacity=0][fill={rgb, 255:red, 120; green, 100; blue, 200 }  ,fill opacity=1 ] (509,82) -- (579,82) -- (579,122) -- (509,122) -- cycle ;
%Straight Lines [id:da6906207774369013] 
\draw    (476,100.5) -- (505,100.5) ;
\draw [shift={(507,100.5)}, rotate = 180] [color={rgb, 255:red, 0; green, 0; blue, 0 }  ][line width=0.75]    (6.56,-1.97) .. controls (4.17,-0.84) and (1.99,-0.18) .. (0,0) .. controls (1.99,0.18) and (4.17,0.84) .. (6.56,1.97)   ;
%Straight Lines [id:da18062367993032136] 
\draw    (580,101.5) -- (609,101.5) ;
\draw [shift={(611,101.5)}, rotate = 180] [color={rgb, 255:red, 0; green, 0; blue, 0 }  ][line width=0.75]    (6.56,-1.97) .. controls (4.17,-0.84) and (1.99,-0.18) .. (0,0) .. controls (1.99,0.18) and (4.17,0.84) .. (6.56,1.97)   ;
%Shape: Circle [id:dp521014707856173] 
\draw  [draw opacity=0][fill={rgb, 255:red, 240; green, 78; blue, 64 }  ,fill opacity=1 ] (616,101.5) .. controls (616,92.39) and (623.39,85) .. (632.5,85) .. controls (641.61,85) and (649,92.39) .. (649,101.5) .. controls (649,110.61) and (641.61,118) .. (632.5,118) .. controls (623.39,118) and (616,110.61) .. (616,101.5) -- cycle ;

% Text Node
\draw (115,90) node [anchor=north west][inner sep=0.75pt]   [align=left] {$\mW^{(1)}$};
% Text Node
\draw (40.5,94.5) node [anchor=north west][inner sep=0.75pt]   [align=left]  {$\vx$};
% Text Node
\draw (321.5,90) node [anchor=north west][inner sep=0.75pt]   [align=center] {$\mW^{(2)}$};
% Text Node
\draw (528,90) node [anchor=north west][inner sep=0.75pt]   [align=left] {$\vw^{(3)}$};
% Text Node
\draw (215.5,89) node [anchor=north west][inner sep=0.75pt]   [align=left] {ReLU}; 

% Text Node
\draw (422.5,89) node [anchor=north west][inner sep=0.75pt]   [align=left] {ReLU};

% Text Node
\draw (615,92) node [anchor=north west][inner sep=0.75pt]   [align=left] {\scriptsize $f(\vx)$};
% Text Node
\draw (588,66.5) node [anchor=north west][inner sep=0.75pt]   [align=center] {$z^{(3)}$};
% Text Node
\draw (482.5,66.5) node [anchor=north west][inner sep=0.75pt]   [align=center] {$\hat{\vz}^{(2)}$};
% Text Node
\draw (379,66.5) node [anchor=north west][inner sep=0.75pt]   [align=center] {$\vz^{(2)}$};
% Text Node
\draw (275,66.5) node [anchor=north west][inner sep=0.75pt]   [align=center] {$\hat{\vz}^{(1)}$};
% Text Node
\draw (171.5,66.5) node [anchor=north west][inner sep=0.75pt]   [align=center] {$\vz^{(1)}$};

\end{tikzpicture}
    \caption{Alternative representation of toy DNN of Figure \ref{fig:toyDNN}.}
    \label{fig:alternativeDNN}
\end{figure}

 We note that $\hat{\vz}^{(2)}$ and $\hat{\vz}^{(1)}$ contain non-linear activation functions (ReLU), and we have to linearize them to keep the linear relationship between the output and these hidden layers. To this end, we create $2 \times \#$ReLU layers (for the lower and upper bound, respectively) diagonal matrices $\underline{\mD}^{(2)},\overline{\mD}^{(2)}, \underline{\mD}^{(1)}, \overline{\mD}^{(1)}$ and bias vectors $\underline{\vb}^{(2)}, \overline{\vb}^{(2)}, \underline{\vb}^{(1)}, \overline{\vb}^{(1)}$ reflecting the impact of each ReLU nodes on the final output. We report for simplicity here the original definition provided in \citep{crown} 
 %also reported in Sec.\;\ref{sec:preliminaries} 
 (a similar definition is applied to compute the $i$-th layer $\overline{\mD}^{(i)}$ and $\overline{\vb}^{(i)}$):

\begin{equation}\label{eq:matrix_bias}
    \underline{\mD}^{(i)} = \begin{cases} 
     1 \quad & \vl_j \geq 0,\\
    0 \quad & \vu_j \leq 0,\\
    \alpha_j  \quad & \vu_j > 0 > \vl_j \text{ and }  \mA_j^{(i+1)} \geq 0,\\
    \frac{\vu_j}{\vu_j - \vl_j} \quad &\vu_j > 0 > \vl_j \text{ and } \mA_j^{(i+1)} < 0
\end{cases}
\hspace{1cm}
\underline{\vb}^{(i)} =\begin{cases} 
    0 \quad &\vl_j > 0 \text{ or } \vu_j \leq 0,\\
    0 \quad &\vu_j > 0 > \vl_j \text{ and } \mA_j^{(i+1)} \geq 0,\\
    -\frac{\vu_j\vl_j}{\vu_j - \vl_j} \quad &\vu_j > 0 > \vl_j \text{ and } \mA_j^{(i+1)} < 0 .
\end{cases}
 \end{equation}

In the following, for simplicity, we always set $\alpha_j=0$. After defining the $i$-th diagonal matrix, we can compute the $i$-th layer relaxation with respect to the output as
 $\underline{\mA}^{(i)} = \underline{\mA}^{(i+1)} \underline{\mD}^{(i)} \mW^{(i)}$ and similarly for the $\overline{\mA}^{(i)}$. In the beginning, we set $\underline{\mA}^{(4)}=\overline{\mA}^{(4)}=I$ and $\underline{\mA}^{(3)}=\overline{\mA}^{(3)}={\vw^{(3)}}^T$ and write starting from right to left (backward computation)\footnote{We report the lower bound version but for the upper we have similar consideration with the reversed inequality.}

\begin{align*}
    f(x) &= z^{(3)}(\vx)\\
         &= {\vw^{(3)}}^T\hat{z}^{(2)}(\vx)\\
         &\geq 
         \underline{\mA}^{(3)}\underline{\mD}^{(2)} z^{(2)}(\vx) && \text{computing a linearization for }\hat{z}^{(2)}\\
         & \geq \underbrace{\underline{\mA}^{(3)}\underline{\mD}^{(2)} \mW^{(2)}}_{\underline{\mA}^{(2)}} \hat{z}^{(1)}(\vx) && \text{rewriting } {z}^{(2)}=\mW^{(2)}\hat{z}^{(1)}\\
         &\geq \underline{\mA}^{(2)}\underline{\mD}^{(1)}z^{(1)}(\vx) && \text{computing a linear bound for }\hat{z}^{(1)}\\
         &\geq \underbrace{\underline{\mA}^{(2)}\underline{\mD}^{(1)}\mW^{(1)}}_{\underline{\mA}^{(1)}}(\vx) && \text{rewriting } {z}^{(1)}=\mW^{(1)}\hat{z}^{(0)}=\mW^{(1)}(\vx)\\
         & \geq \underline{\mA}^{(1)}(\vx) + \underline{d}.
\end{align*}

Hence, in order to linearize $\hat{z}^{(2)} (\vx)$ we compute $\underline{\mD}^{(2)},\overline{\mD}^{(2)}$ and $\underline{\vb}^{(2)}, \overline{\vb}^{(2)}$ which presely correspond to

\begin{minipage}{0.49\linewidth}
    \begin{align*}
        &\underline{\mD}^{(2)} = \begin{bmatrix}
        \frac{u}{u-l} & 0\\
        0 & 1
    \end{bmatrix}= \begin{bmatrix}
        0.4375 & 0\\
        0&1
    \end{bmatrix}\\
     &\overline{\mD}^{(2)} = \begin{bmatrix}
        \alpha & 0\\
        0 & 1
    \end{bmatrix}= \begin{bmatrix}
        0 & 0\\
        0&1
    \end{bmatrix}
    \end{align*}
\end{minipage}\hfill
\begin{minipage}{0.49\linewidth}
\begin{align*}
&\underline{\vb}^{(2)} = \begin{bmatrix}
    \frac{-ul}{u-l}\\
    0 
\end{bmatrix}= \begin{bmatrix}
    15.75\\
    0
\end{bmatrix}\\
&\overline{\vb}^{(2)} = \begin{bmatrix}
    0\\
    0 
\end{bmatrix}
\end{align*}
\end{minipage}

where $\underline{\mD}_{j,j}^{(2)}$ element is computed looking at each intermediate pre-activated bounds of $\vz^{(2)}_j$ and the sign of $j$-th element of the vector $\underline{\mA}^{(3)}$. Thus we have $\underline{\mA}^{(2)} = \underline{\mA}^{(3)} \underline{\mD}^{(2)} \mW^{(2)} = [-1.5, 2.75]$ and $\overline{\mA}^{(2)} = \overline{\mA}^{(3)} \overline{\mD}^{(2)} \mW^{(2)} = [2, 1]$. We proceed computing the diagonal matrix $\underline{\mD}^{(1)}$, $\overline{\mD}^{(1)}$ and bias vectors $\underline{\vb}^{(1)}$, $\overline{\vb}^{(1)}$ for $\hat{\vz}^{(1)}$. In detail, we obtain,

\begin{minipage}{0.49\linewidth}
    \begin{align*}
        &\underline{\mD}^{(1)} = \begin{bmatrix}
        \frac{u}{u-l} & 0\\
        0 & \alpha
    \end{bmatrix}= \begin{bmatrix}
        0.583 & 0\\
        0& 0
    \end{bmatrix}\\
     &\overline{\mD}^{(1)} = \begin{bmatrix}
        \frac{u}{u-l} & 0\\
        0 & \frac{u}{u-l}
    \end{bmatrix}= \begin{bmatrix}
        0.583 & 0\\
        0& 0.643
    \end{bmatrix}
    \end{align*}
\end{minipage}\hfill
\begin{minipage}{0.49\linewidth}
\begin{align*}
&\underline{\vb}^{(1)} = \begin{bmatrix}
    \frac{-ul}{u-l}\\
    0 
\end{bmatrix}= \begin{bmatrix}
    2.92\\
    0
\end{bmatrix}\\
&\overline{\vb}^{(1)} = \begin{bmatrix}
    \frac{-ul}{u-l}\\
    \frac{-ul}{u-l} 
\end{bmatrix}
= \begin{bmatrix}
    2.92\\
    6.43 
\end{bmatrix}
\end{align*}
\end{minipage}

with $\underline{\mA}^{(1)} = \underline{\mA}^{(2)}\underline{\mD}^{(1)}\mW^{(1)} = [-1.75, -0.875]$ and $\overline{\mA}^{(1)} = \overline{\mA}^{(2)}\overline{\mD}^{(1)}\mW^{(1)} = [0.40, 3.74]$.

Finally, we compute the sum in the bias vectors $\underline{d} = \underline{\mA}^{(3)} \underline{\vb}^{(2)} + \underline{\mA}^{(2)} \underline{\vb}^{(1)} = -35.88$ and $\overline{d} = \overline{\mA}^{(3)} \overline{\vb}^{(2)} + \overline{\mA}^{(2)} \overline{\vb}^{(1)} = 12.27$.
The final linear relation is thus $\underline{f}(\vx) \geq \underline{\mA}^{(1)} (\vx) + \underline{d}$ and $\overline{f}(\vx) \leq \overline{\mA}^{(1)} (\vx) + \overline{d}$. Using Hölder’s inequality \citep{crown}, we obtain

\begin{align*}
    \underline{f}_{\text{CROWN}} = \min\limits_{\vx \in \mathcal{C}} \underline{\mA}^{(1)} (\vx) + \underline{d} &= -\vert\vert \underline{\mA}^{(1)} \vert\vert_1 \cdot \varepsilon + \underline{\mA}^{(1)} \vx_0 + \underline{d}\\
    &= -5.25 -0.875 -35.88 = -42. 
\end{align*}
\begin{align*}
    \overline{f}_{\text{CROWN}} = \max\limits_{\vx \in \mathcal{C}} \overline{\mA}^{(1)} (\vx) + \overline{d} &= \vert\vert \underline{\mA}^{(1)} \vert\vert_1 \cdot \varepsilon + \overline{\mA}^{(1)} \vx_0 + \overline{d}\\
    &= 8.28 + 3.74 + 12.27 = 24.3. 
\end{align*}

As we can notice, we obtain a tight over-approximation of the true lower and upper bounds, significantly improving the bounds derived with naive IBP ([-56,32]).

\section{Probabilistically Tightened LiRPA via Underestimation}\label{sec:method}

In this section, we theoretically investigate whether and how it is possible to compute tight intermediate reachable sets, which directly impact the linear output bound computation. As highlighted in Alg.~\ref{alg:pt_lirpa}, the entire linearization process crucially depends on the bounds computed at line 3. However, computing exact values for such bounds is generally infeasible, as the problem has been shown to be NP-hard \cite{Reluplex,fastlin}. Motivated by the speculation of \citet{acrown}, we therefore explore efficient alternatives for approximating the exact (unknown) values of these intermediate bounds as tightly as possible. To this end, we study the impact of a sampling-based approach and what type of probabilistic guarantees can be achieved with it. In detail, we begin employing the \textit{statistical prediction of tolerance limits} results \cite{wilks}, which allows a closed-form derivation of the required sample size to achieve a desired confidence level. However, as we will show, this method only quantifies the fraction of the perturbation region where the guarantees may fail, without addressing the magnitude of potential violations relative to the estimated bounds. Consequently, the resulting probabilistic certificates are weaker than those provided by related approaches such as \citep{proven,randomizedSmoothing}, where guarantees hold across the entire perturbation region. To mitigate this limitation, we first introduce a qualitative extension of Wilks’ guarantees adapted to our setting. While novel, this bound can become overly loose in high-dimensional scenarios. To address this issue, we then develop a new theoretical and practical result based on \textit{extreme value theory} \cite{haan2006extreme}, which allows us to tightly bound the magnitude of possible violations with respect to the estimated bounds.

We now present all the theoretical and practical components to compute tightened bounds with probabilistic guarantees within our \texttt{PT-LiRPA} approach. Our approach is based on a statistical methodology that allows us to compute estimates on the neural network's output in the form $\min_{\vx \in \mathcal{C}} f(\vx)$ and $\max_{\vx \in \mathcal{C}} f(\vx)$, which hold with high confidence.
In particular, we employ the tools of \textit{statistical prediction of tolerance limits} of \citep{wilks}. 
Fix a node $z^{(i)}_j$ in the neural network and, as before, let 
$z^{(i)}_j(\vx)$ be the preactivation value for the node when the input to the network is a vector $\vx.$ Let $\vx_1, \dots \vx_n$ be $n$ points (vectors) independently and uniformly sampled from the perturbation set of interest $\mathcal{C}.$
We compute $z^{(i)}_j$'s (estimated) pre-activated bounds as:
\[\overline{l}^{(i)}_j = \min\limits_{k= 1,\dots, n} z^{(i)}_j(\vx_k); \qquad \underline{u}^{(i)}_j = \max\limits_{k = 1,\dots, n} z^{(i)}_j(\vx_k), \]
i.e., the minimum and maximum value observed from the propagation of the $n$ random points in that specific neuron $z^{(i)}_j$. 
Let ${l^*_j}^{(i)} = \min_{\vx \in \mathcal{C}} z^{(i)}_j(\vx)$ and 
${u^*_j}^{(i)} = \max_{\vx \in \mathcal{C}} z^{(i)}_j(\vx).$ Clearly, we have 
\begin{equation} \label{equation_bound_sampling}
\overline{l}_j^{(i)} \geq {{l}_j^*}^{(i)}; \quad \underline{u}_j^{(i)} \leq {{u}_j^*}^{(i)}.
\end{equation}

However, based on the results of \citep{wilks}, for any $\psi, R \in (0,1)$, we can choose the sample size $n$ that guarantees that the estimated reachable set is correct with probability $\psi$ for at least a fraction $R$ of points in the perturbation set $\mathcal{C}$. Crucially, this statistical result does not require any knowledge of the probability distribution governing the output of our function of interest. %and thus also applies to our setting. 
Formally, we have the following.

\def\l{\overline{l}}
\def\u{\underline{u}}

\begin{lemma}[Probabilistically tightened reachable sets]\label{lemma:opt_bounds_wilks}
    Let $n$ the number of samples employed in the computation and the interval $[\overline{l}^{(i)}_j, \underline{u}^{(i)}_j]$, where $\overline{l}_j^{(i)}$ and $\underline{u}_j^{(i)}$ are the minimum and maximum pre-activation values observed in the sample, respectively. Fix $R \in (0,1)$, then for any further possibly infinite sequence of samples from $\mathcal{C}$, the probability that $[\overline{l}_j^{(i)}, \underline{u}_j^{(i)}]$ is correct\footnote{In the sense that there exists $C'\subseteq C$ such that $|C'|/|C| \geq R$ and 
    for all $\vx \in C'$ it holds that $z^{(i)}_j(\vx) \in [\overline{l}_j^{(i)}, \underline{u}_j^{(i)}].$} for at least a fraction $R$ of points is at least $\psi=n \cdot \int_R^1 x^{n-1}\;dx = (1 - R^n)$. 
\end{lemma}

Hence, following Lemma \ref{lemma:opt_bounds_wilks}, whose proof follows from \citep{wilks}, for any desired confidence level $\psi$, and lower bound fraction $R$, we can compute the number $n$ of samples sufficient to obtain the provable probabilistic guarantees on the desired reachable set accuracy. Specifically, we have that if we use a sample of $n\geq\frac{\ln{(1-\psi)}}{\ln{(R)}}$ input points and with them we obtain an estimated reachable set  $[\overline{l}^{(i)}_j, \underline{u}^{(i)}_j]$, then with probability $\psi$ at most a fraction $(1-R)$ of points in an indefinitely larger future sample could fall outside such reachable set. By taking into account the total number of neurons in the DNN, and using a independently chosen set of points for each neuron, we can extend the above result to compute  reachable sets for each neuron that are simultaneously correct with probability $\psi$ for at least a fraction $R$ of the points in $\mathcal{C}$: 
%$\psi$ determine the number of samples needed to ensure, with a given level of confidence, that all reachable sets throughout the network are correct: %In fact, we have:\\

\begin{proposition}\label{prop:wilks_union_probability}
    Consider an N-layer ReLU DNN with $m$ neurons. Fix a confidence level $\psi$ and coverage ratio $R \in (0,1)$. For each intermediate neuron $z^{(i)}_j,$ collect $n'$ i.i.d. samples $\vx_1, \ldots,\vx_n$ from the perturbation region $\mathcal{C}$, and  compute the approximate reachable set as $\overline{l}_j^{(i)} = \min\limits_{k= 1,\dots, n'} z^{(i)}_j(\vx_k)$ and $\underline{u}_j^{(i)} = \max\limits_{k= 1,\dots, n'} z^{(i)}_j(\vx_k)$. If the number of samples  used for each neuron satisfies $n'\geq \frac{\ln{(1-\psi^{1/m})}}{\ln{(1- (1-R)/m)}}$, then for any (possibly infinite) sequence $X$ of inputs sampled independently and uniformly from $\mathcal{C},$ for each neuron $z^{(i)}_j$ with probability at least $\psi^{1/m}$ there is a subsequence $X'$ of $X$ of size $|X'| \geq 
    \left(1- \frac{1-R}{m}\right)|X|$ such that for each $\vx \in X'$ the estimated reachable set is sound, i.e., 
    $z^{(i)}_j(\vx) \in [\l^{(i)}_j, \u^{(i)}_j].$
    %$ of any further possibly infinite sequence of samples from $\mathcal{C}$.
\end{proposition}

We will now show that we can combine the (probabilistically valid) estimation computed on the reachable sets of individual neurons using a LiRPA approach, so as to obtain, first, a probabilistically valid over-approximation of any ReLU layer in the DNN and thus on the final network's lower and upper outputs. In detail, we begin by proving that the estimated reachable sets used to produce the vectors $\mA$ and $\vb$, together with the linearization applied to each ReLU layer of the network (as in Sec. \ref{sec:lirpa_preliminaries}), yield a probabilistically sound over-approximation that covers at least a fraction $R$ of the perturbation region $\mathcal{C}$.

\begin{lemma}[ReLU Layer Relaxation using \texttt{PT-LiRPA}]\label{lemma:sound_sampling}
    Consider a ReLU DNN with $m$ neurons distributed over $N$ layers, and $\mathcal{C}$ a perturbation region of interest. Fix confidence and coverage parameter $\psi, R \in (0,1)$.
    Fix a layer $i \in \{1, \dots, N-2\},$ and for each $j=1, \dots, d_i$ compute an estimated reachable set $[\overline{l}_j, \underline{u}_j]$ for neuron $z^{(i)}_j,$ using 
    $n' \geq m \frac{ln(1-\psi^{1/m})}{ln(1-(1-R)/m)}$ independently and uniformly sampled point from $\mathcal{C},$ as by Proposition \ref{prop:wilks_union_probability}. 
   Let $\overline{\vl} = (\overline{l}_1, \dots, \overline{l}_{d_i})$ and $\underline{\vu}=(\underline{u}_1, \dots, \underline{u}_{d_i}).$ 
    Let ${\mA}, \vb$ be the vectors of the linear bounds coefficients and biases (inductively) computed for the ReLU layer $i+1$, and such that, with probability $\geq \displaystyle{\psi^{\frac{d_{i+1}+ d_{i+2}+ \cdots + d_N}{m}},}$
    for any $n \in \mathbb{N}$ in any sequence of $n$ points uniformly and independently sampled from $\mathcal{C},$ for at least $n \times \left(1 - \frac{1-R}{m} \sum_{j=i+1}^N d_j\right)$ points $\vx$ in $X$ it holds that  
    \begin{equation} \label{induction_hypothesis} 
    f(\vx) \geq {\mA}^T ReLU(\vv_{\vx}) +\vb, %\geq {\mA^*}^T (\underline{\mD}^*\vv_{\vx} + \underline{\vb}^*),
    \end{equation}
    where $\vv_{\vx}$ is the vector of pre-activation values
    of the neurons in the layer $i$ when the input is $\vx.$
    Then, 
    \begin{enumerate}
        \item with probability %$\geq \psi$ 
    $\geq \displaystyle{\psi^{\frac{d_{i}}{m}},}$
    it holds that 
    for any $n \in \mathbb{N}$ in any sequence $X$ of $n$ points uniformly and independently sampled from $\mathcal{C},$ for at least $n \times \left(1 - \frac{1-R}{m} d_i\right)$ points $\vx$ in $X$ it holds that:
    \begin{equation} \label{induction_step}
    %\small
    {\mA}^T ReLU(\vv_{\vx}) \geq {\mA}^T(\underline{\mD}^*\vv_{\vx} + \underline{\vb}^*)
    \end{equation}

    where 
   
     \begin{center}
        
    \resizebox{0.8\linewidth}{!}{%
    $\underline{\mD}^* = \begin{cases} 
     1 \quad & \overline{\vl}_j \geq 0,\\
    0 \quad & \underline{\vu}_j \leq 0,\\
    \alpha_j  \quad & \underline{\vu}_j > 0 > \overline{\vl}_j \text{ and }  {\mA_j} \geq 0,\\
    \frac{\underline{\vu}_j}{\underline{\vu}_j - \overline{\vl}_j} \quad &\underline{\vu}_j > 0 > \overline{\vl}_j \text{ and } {\mA_j} < 0
    \end{cases}$
    \quad
    $\underline{\vb}^* =\begin{cases} 
    0 \quad &\overline{\vl}_j > 0 \text{ or } \underline{\vu}_j \leq 0,\\
    0 \quad &\underline{\vu}_j > 0 > \overline{\vl}_j \text{ and } {\mA_j} \geq 0,\\
    -\frac{\underline{\vu}_j\overline{\vl}_j}{\underline{\vu}_j - \overline{\vl}_j} \quad &\underline{\vu}_j > 0 > \overline{\vl}_j \text{ and } {\mA_j} < 0 .
    \end{cases}$
    }
    \end{center}

    \item 
    with probability
    $\geq \displaystyle{\psi^{\frac{d_{i}+ d_{i+1}+ \cdots + d_N}{m}}}$, for vectors ${\mA}'= {\mA}^T \mD^*$ and $\vb'= \vb + \mA^T\vb^*$  it holds that for any $n \in \mathbb{N}$ in any sequence $X$ of $n$ points uniformly and independently sampled from $\mathcal{C},$ for at least
    $n \times \left(1 - \frac{1-R}{m} \sum_{j=i}^N d_j\right)$ points $\vx$ in $X$ we have that  
    \begin{equation} \label{induction-result}
    f(\vx) \geq {\mA'} ReLU(\vv_{\vx}) +\vb', %\geq {\mA^*}^T (\underline{\mD}^*\vv_{\vx} + \underline{\vb}^*),
    \end{equation}

    \end{enumerate}

    \end{lemma}

\begin{proof}
The proof closely follows the analogous result at the basis of the LiRPA approach of \citet{crown}. 
The only (crucial) difference is that we construct  $\underline{\mD}^*$ and $\underline{\vb}^*$, i.e., the diagonal matrix, and the bias vector meant to provide a linear lower bound on the post-activation values of layer $i,$ using the vectors of estimated reachable sets of the nodes in the layer, instead of their actual reachable sets.

Let $X$ be a sequence of $n$ points independently and uniformly sampled from $\mathcal{C}.$

\def\E{\mathcal{E}}

For each $j=1, \dots, d_i,$ let $$X_{j} = \{\vx \in X \mid z^{(i)}_j(\vx) \not \in [\l_j, \u_j ]\},$$
and let $\mathcal{E}_j$ be the event
$$\E_j = \{|X_j| \leq \frac{1-R}{m}|X|\}.$$

Because of the way we compute the estimated reachable sets $[\l_j,\u_j]$ (Proposition \ref{prop:wilks_union_probability}) we have that for each $j$
$$Pr[\E_j] \geq \psi^{1/m}.$$
Moreover, these events are independent, hence
$$Pr[\E_1 \wedge \E_2 \wedge \cdots \wedge \E_{d_i}] \geq \psi^{d_i/m}.$$

Let $$X_{OK} = X\setminus\left(\bigcup_{j=1}^{d_i} X_j\right)=
\{\vx \in X \mid \forall j, z^{(i)}_j \in [\l_j, \u_j]\}.$$
Then, $$|X_{OK}| \geq |X| - \sum_{j=1}^{d_i} |X_j|,$$ and in particular, if for each $j$ it holds that 
$|X_j| \leq \frac{1-R}{m}|X|$--i.e., when the 
event $\E_1 \wedge \cdots \wedge \E_{d_i}$ occurs--we have $|X_{OK}| \geq \left(1 - \frac{(1-R)}{m}d_i \right) |X|.$

We can conclude that the probability of the event $\E_{OK} = \{|X_{OK}| \geq (1-\frac{(1-R)d_i}{m})|X|\}$
satisfies
$$Pr[\E_{OK}] \geq Pr[\E_1 \wedge \E_2 \wedge \cdots \wedge \E_{d_i}] \geq \psi^{d_i/m}.$$

In words, we have shown that with probability 
$\psi^{1/m}$ for each point $\vx$ in a fraction of $X$ of size  
$(1-\frac{(1-R)d_i}{m})|X|,$  for all nodes $z^{(i)}_j$ in layer $i,$ the pre-activation value $z^{(i)}_j(\vx)$ induced by $\vx$ is contained in the estimated reachable sets 
$[\l_j, \u_j].$

In order to simplify the notation, let us fix an $\vx$ from $X_{OK}$ and let $\vv = \vz^{(i)}(\vx)$ be the vector of pre-activation values induced by $\vx$ in the nodes of layer $i.$ 
We will show that
for the $\underline{\mD}^*$ and $\underline{\vb}^*$ 
defined in the statement, we have  
that for each $j$ it holds that 
${\mA_j} (ReLU (\vv_j)) \geq {\mA_j} (\underline{\mD}^*_{j,j} \vv_j +\underline{\vb}^*_j),$ which immediately implies
the desired result, since 
${\mA}^T ReLU(\vv)= \sum_j {\mA_j} (ReLU (\vv_j)) \geq 
{\mA_j} (\underline{\mD}^*_{j,j} \vv_j +\underline{\vb}^*_j) = {\mA}^T(\underline{\mD}^*\vv_{\vx} + \underline{\vb}^*).$

    %Following the construction is in any standard LiRPA approach, we use ${\mA_j^*}$ which is the $j$-th coefficient of the row vector $\mA^*$ that encodes the linear relation (inductively assumed to hold with the desired statistical guarantees) between the layer $i+1$ and the output, and instead of a sound reachable set $[\vl_j, \vu_j],$ typically computed via IBP, we use $[\overline{\vl}_j, \underline{\vu}_j]$ which are the estimated reachable sets of the $j$-th node, computed with the sampling-based approach. 
    
    We start by noticing that 
   % for unstable ReLU nodes, 
    the following inequalities hold
    \begin{equation}\label{eq:inequality}
        \vl_j \leq \vl^*_j \leq \overline{\vl}_j < 0 < \underline{\vu}_j \leq \vu^*_j \leq \vu_j,
    \end{equation}
    where $[\vl^*_j, \vu^*_j]$ is the actual reachable set for the $j$-th node of the layer under consideration. 
    
    We split the argument into three cases according to the sign of the estimated values $\l_j$ and $\u_j.$ Moreover, we split each case into subcases according to the sign of $\mA_j$ and $\vv_j.$\\

    \noindent\fbox{{\sc Case 1.} $\underline{\vu}_j > 0 > \overline{\vl}_j$}
     \vspace{2mm}

    From inequality \ref{eq:inequality}, we know that since we are underestimating true bounds $[\vl^*_j, \vu^*_j]$, the ReLU node is actually unstable, even for any LiRPA approach. Comparing the diagonal coefficients $\frac{\underline{\vu}_j}{\underline{\vu}_j - \overline{\vl}_j}$ with $\frac{\vu_j}{\vu_j - \vl_j}$ and the biases $-\frac{\underline{\vu}_j\overline{\vl}_j}{\underline{\vu}_j - \overline{\vl}_j}$ with $-\frac{\vu_j\vl_j}{\vu_j - \vl_j}$ of \texttt{PT-LiRPA} and any LiRPA cannot be helpful. The relation between the coefficients strongly depends on the quality of the bounds computed, and we cannot draw any direct conclusion since in some cases $\underline{\mD} > \underline{\mD}^*$ and in some cases not. Hence, we need to proceed by subcases.

    \underline{{\em Subcase 1.1.} $\mA_j<0$ and $\vv_j < 0.$}

    Since $\vv_j < 0$ then $ReLU (\vv_j) = 0.$
    Moreover, we have 
    $$  
    (\underline{\mD}^*_{j,j} \vv_j +\underline{\vb}^*_j) = 
    \frac{\underline{\vu}_j}{\underline{\vu}_j - \overline{\vl}_j} \vv_j + \Big(-\frac{\underline{\vu}_j\overline{\vl}_j}{\underline{\vu}_j - \overline{\vl}_j} \Big)\\
    = \frac{\underline{\vu}_j (\vv_j - \overline{\vl}_j)}{\underline{\vu}_j - \overline{\vl}_j} \geq 0,$$
    where the last inequality holds    
    since $\vv_j \in [\overline{\vl}_j, \u_j]$ implies $\vv_j - \overline{\vl}_j \geq 0.$
        Multiplying both sides by $\mA_j<0$ and using $ReLU (\vv_j) = 0,$ we get 
    $${\mA_j} (ReLU (\vv_j)) = 0 \geq {\mA_j} (\underline{\mD}^*_{j,j} \vv_j +\underline{\vb}^*_j)$$ 
    %to prove becomes $ReLU (\vv_j) \leq \underline{\mD}^*_{j,j} \vv_j +\underline{\vb}^*_j$ where if $\vv_j < 0$ we have:
%
 %  $$0 \leq \frac{\underline{\vu}_j}{\underline{\vu}_j - \overline{\vl}_j} \vv_j + \Big(-\frac{\underline{\vu}_j\overline{\vl}_j}{\underline{\vu}_j - \overline{\vl}_j} \Big)\\
  %  = \frac{\underline{\vu}_j (\vv_j - \overline{\vl}_j)}{\underline{\vu}_j - \overline{\vl}_j}
   % $$

    %since $\overline{\vl}_j < 0$ and $\overline{\vl}_j \leq \vv_j$, thus $\vv_j - \overline{\vl}_j \geq 0$ which is enough to prove the inequality.

 \underline{{\em Subcase 1.2.} $\mA_j<0$ and $\vv_j \geq 0.$}

    Since $\vv_j \geq 0$ then $ReLU (\vv_j) = \vv_j.$
    Moreover, we have 
    $$  
    (\underline{\mD}^*_{j,j} \vv_j +\underline{\vb}^*_j) = 
    \frac{\underline{\vu}_j}{\underline{\vu}_j - \overline{\vl}_j} \vv_j + \Big(-\frac{\underline{\vu}_j\overline{\vl}_j}{\underline{\vu}_j - \overline{\vl}_j} \Big)\\
    \geq \vv_j,$$
    where the last inequality holds    
    since under the standing hypothesis the coefficient of $\vv_j$ in the left hand side is $\geq 1$ and term 
    $-\frac{\underline{\vu}_j\overline{\vl}_j}{\underline{\vu}_j - \overline{\vl}_j}$ is non-negative.
Multiplying both sides by $\mA_j < 0$ and using $ReLU (\vv_j) = \vv_j,$ we get 
    $${\mA_j} (ReLU (\vv_j)) = {\mA_j} \vv_j \geq {\mA_j} (\underline{\mD}^*_{j,j} \vv_j +\underline{\vb}^*_j).$$

    %where $\overline{\vl}_j<0$ and $(\vv_j - \underline{\vu}_j) < 0$. The inequality necessarily holds since the numerator and the denominator are positive. This concludes the first subcase.

    \underline{Subcase 1.3. $\mA_j\geq0$ and $\vv_j < 0$.} 
    
    We have 
    $$ReLU (\vv_j) = 0  \geq \alpha_j \vv_j + 0 = 
    \underline{\mD}^*_{j,j} \vv_j +\underline{\vb}^*_j\;,$$  since $0 <\alpha_j.$ Hence,  
    multiplying both sides by $\mA_j \geq 0$ we obtain again the desired inequality.
    \underline{Subcase 1.4. $\mA_j\geq0$ and $\vv_j \geq 0$.} 
    
    We have 
    $$ReLU(\vv_j) = \vv_j \geq \alpha_j \vv_j + 0 = 
    \underline{\mD}^*_{j,j} \vv_j +\underline{\vb}^*_j$$
    since $\alpha < 1.$ Again,  
    multiplying both sides by $\mA_j \geq 0$ we obtain again the desired inequality.
    This concludes the first case.

    %For the next cases, $\overline{\vl}_j > 0$ and  $\underline{\vl}_j > 0$, from the inequality \ref{eq:inequality}, we could define a ReLU node as (un)stable when, actually, it is not. However, from Lemma \ref{lemma:opt_bounds_wilks}, we know that
    %in even a potentially infinite sampling of points 
    %with probability $\psi^{1/m}$ at most a fraction $(1-R)/m$ of the points in $\mathcal{C}$ could induce a pre-activation value of the neuron
    %$z^{(i)}_j$ outside the reachable set $[\overline{l}_j, \underline{u}_j]$. Thus, with probability $\geq \psi$.

    \vspace{3mm}
    \noindent\fbox{{\sc Case 2.} $\overline{\vl}_j > 0$}
     \vspace{2mm}

    %\underline{{\em Subcase 2.1.} $\mA_j < 0.$} 
    
    Since $\vv_j \in [\l_j, \u_j],$ with  $\l_j > 0$ we have $ReLU(\vv_j) = \vv_j.$ 
    
    Moreover, we have 
    $$\underline{\mD}^*_{j,j} \vv_j +\underline{\vb}^*_j 
    =  1 \cdot \vv_j + 0 = \vv_j = ReLU(\vv_j)$$ 
    from which, multiplying both sides by $\mA_j$ 
    yields the desired inequality 
    $\mA_j ReLU(\vv_j) \geq \mA_j (\underline{\mD}^*_{j,j} \vv_j +\underline{\vb}^*_j).$

%    \underline{$\mA_j \geq 0.$} We need to show that $ReLU (\vv_j) \geq \underline{\mD}^*_{j,j} \vv_j +\underline{\vb}^*_j$, which for similar previous consideration we have $\vv_j \geq 1 \cdot \vv_j$.
%    This concludes the second case.
    
    \vspace{3mm}
    \noindent\fbox{{\sc Case 3.} $\underline{\vu}_j < 0$}
    \vspace{2mm}

    Since $\vv_j \in [\l_j, \u_j]$ and $\u_j < 0$ we have 
    $ReLU(\vv_j) = 0.$
    Moreover, under the standing hypothesis, we also have 
    $\underline{\mD}^*_{j,j} \vv_j +\underline{\vb}^*_j = 0 = 
    ReLU(\vv_j).$
    Again, multiplying both sides by $\mA_j$, we have that the desired inequality holds also in this case. 
    We have proved that for each $\vx$ in $X_{OK}$ we have $\mA^T ReLU(\vv_j) \geq \mA^T (\underline{\mD}^* \vv_x +\underline{\vb}^*).$ Finally, recalling that with probability $\psi^{d_i/m}$ we have that  
    $|X_{OK}| \geq |X|(1-\frac{1-R}{m}d_i),$ we have that the previous inequality holds with the desired statistical guarantees, which completes the proof of claim (1).\\

    For proving claim (2) it is enough to note that  the guarantees on $\mA, \vb$ and (\ref{induction_hypothesis}) hold independently of the 
    choice of the samples leading to the statistical guarantee on (\ref{induction_step}).
    Hence they simultaneously happen--yielding that  (\ref{induction-result}) holds--with the product of the probabilities of (\ref{induction_hypothesis}) and (\ref{induction_step}), i.e., $\displaystyle{\psi^{\frac{d_{i}+ d_{i+1}+ \cdots + d_N}{m}}}.$ 
    In particular, with this probability, for any $n \in \mathbb{N}$ in any sequence $X$ of $n$ points uniformly and independently sampled from $\mathcal{C},$ neither of (\ref{induction_hypothesis}) and (\ref{induction_step}) fails on at least  $n \times \left(1-\frac{1-R}{m} \sum_{j=i}^N d_j\right)$ points.

\end{proof}

 As a direct implication of this result, we can show that the neural network's output linear bounds computed using \texttt{PT-LiRPA} remain, with high probability, an overestimation of the real lower and upper bound, respectively, for at least a fixed fraction $R$ of the perturbation region under consideration.

\begin{theorem}[\texttt{PT-LiRPA} weak probabilistic guarantees]\label{theorem:PT_lirpa_wilks}
   Fix an N-layer ReLU DNN with $m$ neurons with $f$ the function it computes.
   Then, for any $\psi,  R \in (0,1)$ \texttt{PT-LiRPA},  using a total number of  $n \geq m \frac{\ln(1 - \psi^{1/m})}{\ln(1 - (1-R)/m)}$ independent and uniformly distributed input points, computes a linear approximation $\va^T_{\texttt{PT-LiRPA}}(\vx) + \vc_{\texttt{PT-LiRPA}}$ of $f$ such that with probability at least $\psi$ for at least a fraction $R$ of possibly infinite samples of input points, ${\vx},$ it holds that 
    $$
    %\min_{{\bf x} \in {\cal C}} 
    f({\vx}) \geq \va^T_{\texttt{PT-LiRPA}}(\vx) + \vc_{\texttt{PT-LiRPA}}.$$
\end{theorem}

\begin{proof}
    % From our Lemma \ref{lemma:sound_sampling}, we note that every single node in a ReLU layer is probabilistically soundly overapproximated with a confidence $\geq \psi^{(1/Nk)}$, and at most a fraction of $(1-R)/Nk$ of any further possibly infinite sequence of samples could violate the reachable set of that specific node. Hence, with probability $\geq (\psi^{(1/Nk)})^{Nk}= \psi$ at most a fraction $1-R$ of a further sampling could fall outside at least one of the reachable set computed using the first sample, meaning that with probability $\psi$, for at least a fraction $R$ of a further sample, all the reachable set computed in the first sample are correct. 
    The proof then directly follows from Lemma \ref{lemma:sound_sampling} and the derivations of \citet{crown}.
\end{proof}

\subsection{A Qualitative Bound of Statistical Prediction of Tolerance Limit}\label{sec:formula}

The result of \citep{wilks} allows us to bound the error of the sample-based procedure employed in our \texttt{PT-LiRPA} framework in terms
of the number of potential violations in future samples.
However, 
it does not provide any 
%qualitative 
information on the 
%distance 
magnitude of such possible violations with respect to the estimated bounds. Additionally, the probabilistic guarantees we provide so far are, broadly speaking, weaker than those offered by related probabilistic methods such as \citep{proven,randomizedSmoothing}, where the guarantees hold for any input $\vx \in \mathcal{C}$, and not only for a predefined fraction of it. In this section, we aim to complement \citet{wilks} statistical result with a qualitative interpretation that bounds the potential error between the true minimum and its estimate based on samples. Specifically, we leverage known results on \textit{extreme value theory} \citep{haan2006extreme} to strengthen and extend our original guarantee to apply to the entire perturbation set $\mathcal{C}$, thus bringing the proposed solution in line with other state-of-the-art probabilistic approaches.

We start by noticing that the problem of finding an empirical minimizer close to a function's exact minimizer is well established in the statistical literature (see, e.g., \citep{archetti1984survey}). Exploiting quantitative assumptions on the objective function—such as a Lipschitz condition that holds in our setting—can facilitate the bounding of the possible magnitude error in the sampled-based minimizer computation approach.
%However, as noted in \citep{betro1991bayesian}, the number of function evaluations required to obtain a tight bound using this approach may scale exponentially with the dimension $d$ of the perturbation region $\mathcal{C}$, and the Lipschitz constant $L$.\\
In this vein, we begin by providing a first worst-case qualitative bound on the maximum $\Delta$ error we can achieve when employing our sampling-based approach to compute the estimated reachable set with respect to the real (unknown) ones. For readability, in the following we simplify the notation $z_j^{(i)}$—which denotes the pre-activation value of the node in position $j$ of layer $i$—by removing the indices and referring generically to a node’s pre-activation value as $z$. We adopt the same simplification for the corresponding lower and upper bound estimates, denoted as  $\overline{l}$ and $\underline{u}$, respectively.

We start with the following result.
\newpage

\begin{theorem}[Worst-case excess bound]\label{theorem:worst_case_bound}
 Fix a neuron in our $N$-layer DNN. 
 Let $z$ be the function 
 mapping the input $\vx$ to the network to the pre-activation value  
 $z(\vx)$ of the neuron of interest. 
 Let $\overline{l} = \min_{k=1,\ldots,n} z(\vx_k)$ be 
 the minimum pre-activation value observed in a sample of $n$ inputs, $\vx_1,\ldots,\vx_n$ independently and uniformly drawn from $\mathcal{C} \subseteq \mathbb{R}^{d}$. 
 Let $l^* = \min_{\vx \in \mathcal{C}} z(\vx)$ be the actual minimum pre-activation value achievable for $z$ over all $\vx \in \mathcal{C}$. Then, if $z$ is Lipschitz continuous, it holds that:
    \begin{equation}\label{eq:qualitative_bound1}
        \small
        Pr[\vert \overline{l} - l^* \vert \leq \Delta ] \geq 1 - exp\left({-n\left(\frac{\Delta}{L}\right)^d\cdot \frac{\pi^{d/2}}{\Gamma(\frac{d}{2} + 1)}}\right)
    \end{equation}
    where $L$ is the Lipschitz constant of the function $z()$, $d$ is the dimension of the perturbation region $\mathcal{C}$, $\Gamma$ is the gamma function, and $\Delta$ the maximum error we are interested in bounding.
\end{theorem}

\begin{proof}
    Since z is an $L$-Lipschitz function, then we have that for each $k=\{1,\ldots,n\}$, and $\vx^*$ being a minimizer of $z$, i.e., $z(\vx^*)=l^*$, it holds that
    $$|z(\vx_k)- l^*| \leq L ||\vx_k-\vx^*||_2.$$
    From the function's Lipschitz continuity property, we have
    
    \begin{align*}
        Pr[|\overline{l} - l^*| \geq \Delta] 
        &= Pr[\forall k\; z(\vx_k) - z(\vx^*) \geq \Delta] \\
        &\leq Pr[\forall k\; L ||\vx_k - \vx^*||_2 \geq \Delta]\\
        &=  Pr[\forall k\; ||\vx_k - \vx^*||_2 \geq \frac{\Delta}{L}]
    \end{align*}

    Fix $S = \{\vx \in \mathcal{C} : ||\vx - \vx^*||_2 \leq \Delta / L\}.$
    This is the sphere of radius $\Delta / L$ centered at $\vx^*$. 
    For a uniformly sampled point $\vx \in \mathcal{C},$ let $\mu(S) = Pr[\vx \in S].$ This is equal to the volume of the set $S$ which is proportional to $(\Delta / L)^d \cdot c(d)$, where $c(d) = \frac{\pi^{d/2}}{\Gamma(\frac{d}{2} + 1)}$ is a constant that depends on the dimension $d$. Hence, by the independence of the $\vx_k$s, the probability $Pr[\forall k\;\vx_k \notin S]= (1 - \mu(S))^n$.

    \begin{align*}
        Pr[|\overline{l} - l^*| \geq \Delta] &\leq Pr[\forall k\; ||\vx_k - \vx^*||_2 \geq \frac{\Delta}{L}]\\
        & = Pr[\forall k \; \vx_k \notin S]\\
        & = (1 - \mu(S))^n \\
        & \text{ from the fact that } (1-x)^n \approx e^{-nx}\\
        & \leq exp\left(-n\left(\frac{\Delta}{L}\right)^d\cdot \frac{\pi^{d/2}}{\Gamma(\frac{d}{2} + 1)}\right)
    \end{align*}
    
    Hence,  $Pr[\vert \overline{l} - l^*\vert \leq \Delta] \geq 1 - Pr[\vert \overline{l} - {l^*}\vert \geq \Delta] \geq 1 - exp\left(-n\left(\frac{\Delta}{L}\right)^d\cdot \frac{\pi^{d/2}}{\Gamma(\frac{d}{2} + 1)}\right)$. 
\end{proof}

  \begin{figure}[h!]
    \centering
    \includegraphics[width=0.4\linewidth]{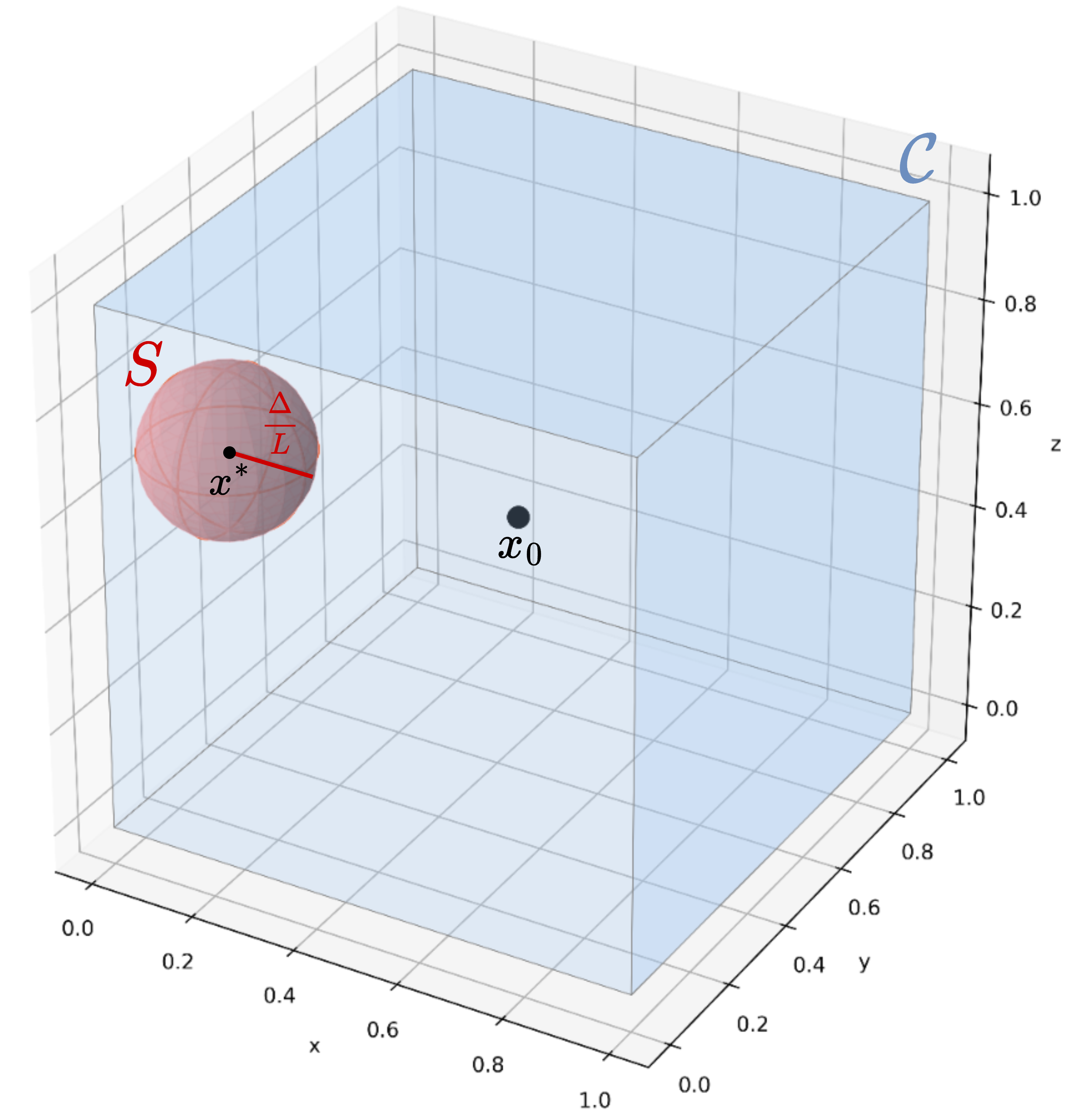}
    \caption{Illustrative representation in 3D of Theorem 3.5.}
    \label{fig:theorem3.5}
\end{figure}

Broadly speaking, in the theorem, we bound the probability that the estimated minimizer deviates from the true minimizer by more than a predefined threshold $\Delta$. Owing to the Lipschitz properties of DNNs, this is equivalent to estimating the probability that a random sample of inputs includes at least one point within distance $\Delta/L$ of the true minimizer $\mathbf{x}^*$, i.e., an input $\vx$ in
$S = \{ \mathbf{x} \in \mathcal{C} : \|\mathbf{x} - \mathbf{x}^*\|_2 \leq \Delta/L \}$.
Geometrically, $S$ corresponds to the intersection of the perturbation region $\mathcal{C}$ with a ball of radius $\Delta/L$ centered at $\mathbf{x}^*$. If $\mathbf{x}^*$ lies sufficiently far from the boundary of $\mathcal{C}$, then $S$ is exactly such a ball. Otherwise, if $\mathbf{x}^*$ is near the boundary, $S$ becomes a truncated ball, i.e., the intersection with $\mathcal{C}$. This is clearly represented in Fig. \ref{fig:theorem3.5}, where we depict a potential perturbation region $\mathcal{C}$ in 3d derived from the original input $\vx_0$ and the minimizer $\vx^*$ of $f$ for $\mathcal{C}$.

Clearly, the bound given in Theorem \ref{theorem:worst_case_bound} accounts for the worst-case scenario and assumes no specific property of the distribution on the input.
When the perturbation region has high dimensionality or the DNN has a large Lipschitz constant, this bound may become too loose to offer meaningful insights. Following EVT, and particularly the results of \citep{de1981estimation}, we can provide a tighter bound on the error in the estimation. 
\newpage
\begin{lemma}[Tighter Qualitative bound]\label{lemma:qual_bound}
Fix a neuron in our $N$-layer DNN and let $z$ be the be the function mapping the input network to the pre-activation value of the neuron. Let $\overline{l} = \min_{k=1,\ldots,n} z(\vx_k)$ the minimum pre-activation value observed in a sample of $n$ inputs, $\vx_1,\ldots,\vx_n$ independently and uniformly drawn from $\mathcal{C} \subseteq \mathbb{R}^{d}$. Let ${l^*} = \min_{\vx \in \mathcal{C}} z(\vx)$, the true minimum pre-activation value achievable over all $\vx \in \mathcal{C}$. Let $Y_1,Y_2,\ldots,Y_n$ be the order statistics for $z(\vx_1),\ldots,z(\vx_n)$. For all $p \in (0,1)$, then it holds that:
    \begin{equation}\label{eq:qualitative_bound1}
        Pr\left[\vert \overline{l} - l^* \vert \leq \frac{Y_2-Y_1}{(1-p)^{-a} - 1} \right] \geq 1 - p
    \end{equation}
    with $a \approx log(\nu)/log\left(\frac{Y_\nu-Y_3}{Y_3-Y_2}\right)$, with $\nu$ any integer valued function of $n$ such that $\nu(n)\to\infty$ and $\nu(n)/n \to 0$.
\end{lemma}

Importantly, Lemma \ref{lemma:qual_bound}, holds under the assumption that the samples $Y_1,\ldots,Y_n$ follow a nondegenerate limit distribution function in the form $1-exp(-x^a)$ for some $a>0$ \citep{de1981estimation}. In particular, this is true for uniformly differentiable $f$, which is the case for most common neural networks, e.g., when the activation functions are Sigmoid, Tanh, etc. For ReLU-based networks, where the differentiability requirement is not satisfied, we can still apply Lemma \ref{lemma:qual_bound} by splitting the non-linearities, i.e., splitting all the ReLU nodes, achieving a linear system differentiable everywhere \cite{babSplitRelu}. 

Concretely, whenever a ReLU’s interval bound crosses 0 (i.e., $\underline{z} < 0 < \overline{z}$) computed over the perturbation region $\mathcal{C}$, we split the computation into two linear branches enforcing $z(\vx)\ge 0$ (active) and $z(\vx)\le 0$ (inactive). On each resulting subregion, the network is affine, so the neuron's pre-activation is an affine (hence differentiable) map of $\vx$. Sampling i.i.d. uniformly from $\mathcal{C}$ and rejecting points that do not satisfy a branch’s linear constraints is equivalent to sampling i.i.d. uniformly from that constrained subregion. Hence, the order statistics restricted to the subregion obey the required extreme-value limit. In practice, in our experimental setting described in Sec. \ref{sec:empirical}, where $\mathcal{C}$ is an $\ell_\infty$-ball and ReLU constraint splitting is applied, the rejection rate remains very low. Specifically, the fraction of discarded samples is always negligible (typically below 5\% on average). This is because the probability that a uniformly sampled point violates the constraints decreases rapidly as the sample size increases, and the rejection step only eliminates points outside the feasible region while preserving the independence of the remaining samples.

% Since the points sampled from the original perturbation region are i.i.d., the rejection phase acts simply as a filter to enforce the constraints, still resulting in a computationally efficient method. 

\subsection{Extending Wilks' Probabilistic Guarantees}
Building on the results of Lemma \ref{lemma:qual_bound}, we can extend the theoretical guarantees of \texttt{PT-LiRPA} to the entire perturbation region $\mathcal{C}$. Specifically, Lemma \ref{lemma:qual_bound} provides a bound on the maximum error in estimating the true minimum (or maximum) of a given intermediate node $z$. By adding this EVT-based error bound to the initial estimates $\overline{l}$ and $\underline{u}$, obtained from the first random sampling of $n$ inputs from $\mathcal{C}$, we derive a probabilistically tight approximation of the reachable set for $z$ that holds for any $\vx \in \mathcal{C}$. Notably, following the asymptotic requirements in \citep{haan2006extreme}, we can choose the number of upper order statistics as $\nu(n) = \lfloor n^{\xi} \rfloor$, with $\xi \in (0,1)$, which clearly satisfies the conditions $\nu(n) \to \infty$ and $\nu(n)/n \to 0$.

Hence, for any intermediate neuron in the network, we have:

\begin{theorem}[Improved \texttt{PT-LiRPA} probabilistic guarantee on the estimated reachable set]\label{theorem:pt_lirpa_guarantees}
    Fix a positive integer $n$ and real values $p, \xi \in (0,1),$ and let $\nu = \lfloor n^\xi \rfloor.$
    For any neuron $z$ in an N-layer DNN, let $z(\vx)$ denote the pre-activation value of $z$ when $\vx$ is the input to the network. Let $\overline{l} = \min_{k=1,\ldots,n} z(\vx_k)$ and 
    $\underline{u} = \max_{k=1,\ldots,n} z(\vx_k)$ as the minimum and maximum pre-activation values observed in a sample of $n$  random inputs $\vx_1,\ldots,\vx_n$ independently and uniformly drawn from $\mathcal{C} \subseteq \mathbb{R}^{d}$. Let $Y_1 \leq Y_2 \leq \cdots \leq Y_n$ be the order statistics for the observed values $z(\vx_k)$. 
    %Fix $p \in (0, 1)$, $\nu = \lfloor n^\xi \rfloor$, with $\xi\in(0,1)$. 
    
    Then for any $\vx \in \mathcal{C}$ it holds that:
    \[ Pr\left[  
        \overline{l} - \frac{Y_2 - Y_1}{(1 - p)^{-a_l} - 1}\leq z(\vx) \leq  
        \underline{u} + \frac{Y_n - Y_{n-1}}{(1 - p)^{-a_u} - 1} 
    \right] \geq 1-2p .\] 

    with $a_l \approx \frac{\log(\nu)}{\log\left( \frac{Y_\nu - Y_3}{Y_3 - Y_2} \right)},
        a_u \approx \frac{\log(\nu)}{\log\left( \frac{Y_{n-2} - Y_{n-\nu}}{Y_{n-1} - Y_{n-2}} \right)}.$ 
\end{theorem}

\begin{proof}
    The proof directly follows from the union bound, exploiting Lemma \ref{lemma:qual_bound}.
\end{proof}

Therefore, in a network with $m$ neurons, using the LiRPA approach employing the estimates $\hat{l}$ 
and $\hat{u}$ for the reachable set of neuron $z$ as given by 
Theorem \ref{theorem:pt_lirpa_guarantees}, i.e., 
\begin{equation} \label{gaps}
\hat{l} = \overline{l} - \frac{Y_2 - Y_1}{(1 - p)^{-a_l} - 1}
\qquad 
\hat{u} = \underline{u} + \frac{Y_n - Y_{n-1}}{(1 - p)^{-a_u} - 1},
\end{equation} 

we can compute a linear lower bound function $\va^T_{\texttt{PT-LiRPA}}(\vx)+ \vc_{\texttt{PT-LiRPA}}$ such that with probability at least $1-2mp$ for any $\vx \in \mathcal{C}$ satisfies: 

%employing the linearization process, we have the final $\texttt{PT-LiRPA}$ probabilistic guarantee on the linear output lower (and upper) bounds as follows.

%\begin{lemma}[\texttt{PT-LiRPA} probabilistic guarantees on the linear output bounds]\label{lemma: PT-LiRPA_bound}
%   Fix an N-layer ReLU DNN and let $f$ denote the function it computes. Then, for any $p \in (0,1)$ \texttt{PT-LiRPA},  using $n$ independent and uniformly distributed input points from the perturbation region $\mathcal{C}$, computes a linear approximation $\va^T_{\texttt{PT-LiRPA}}(\vx) + \vc_{\texttt{PT-LiRPA}}$ of $f$ such that with probability at least $1-2p$,  for any new ${\vx} \in \mathcal{C}$ it holds that 
%    $$
    
\begin{equation} \label{final-inequality}
f({\vx}) \geq \va^T_{\texttt{PT-LiRPA}}(\vx) + \vc_{\texttt{PT-LiRPA}}.
\end{equation}

We apply the sampling procedure independently for each of the $m$ neurons, estimating the reachable set using $n$ i.i.d. samples drawn from $\mathcal{C}$. By Theorem \ref{theorem:pt_lirpa_guarantees}, each reachable set is correct with probability at least $1 - 2p$. Then we use union bound to estimate the probability  that all reachable sets are simultaneously correct, yielding a lower bound $1 - 2mp$. We note that this analysis does not need independence between estimation errors across neurons. In fact, dependencies may arise due to the layered nature of DNNs, which may lead to compounding errors in deeper layers.

Some observations are in order. The result of Theorem \ref{theorem:pt_lirpa_guarantees} provides for any choice of $n$ and $p$ a guarantee holding for all input values $x \in \mathcal{C}.$ However, there is a potential weakness in its practical use since we would also like to guarantee that the two sides of inequality (\ref{final-inequality}) are as close as possible, i.e., that the linear lower bound is as tight as possible. For this, we would like to have that for each neuron $z$ the values $\hat{l}$ (respectively, $\hat{u}$) and $\overline{l}$ (respectively, $\underline{u}$) are close.  
However, the tail index parameters $a_l$ and $a_u,$ ruling their difference, depend on the shape of the tail of the distribution of $z(\vx)$ over $\mathcal{C}.$ This precludes the possibility of computing the minimum value of $n$ yielding to achieve a desired precision for a given desired confidence $p$. 

One possibility to address this issue is to guess the minimum number of samples by a standard doubling technique: keep on doubling the number of samples used until the estimated tail corrections fall below a desired threshold. 
Alternatively, we can start with a conservative sample size $n$ inspired by Wilks' formula and our extension (Proposition \ref{prop:wilks_union_probability}), set $\nu = \lfloor n^{\xi}\rfloor$ with $\xi \in (0,1)$, and compute the resulting values $\hat{l}$, $\hat{u}$ as by Theorem \ref{theorem:pt_lirpa_guarantees}. While potentially suboptimal, the experiments show that this approach produces linear bounds that, besides satisfying the above guarantees, remain significantly tighter than those provided by the traditional LiRPA-based approach, especially in deeper layers where LiRPA errors tend to compound. 

\subsection{Example of \texttt{PT-LiRPA} Linear Bounds Computation.}\label{subsec:example_computation}

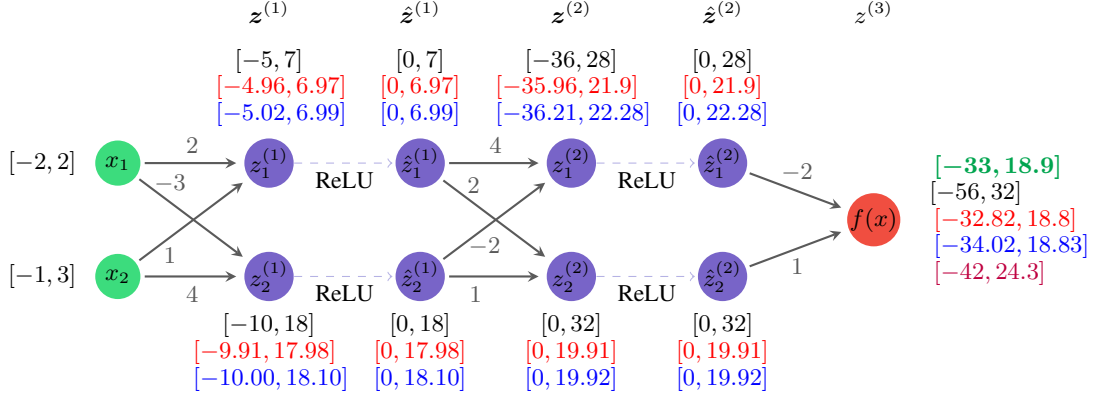
\begin{figure}[h!]
    \centering
    \def\layersep{2.0cm}
\begin{tikzpicture}[shorten >=1pt,->,draw=black!50, node
				distance=\layersep,font=\footnotesize]
				
				\node[input neuron] (I-1) at (0,-1) {$x_1$};
				\node[input neuron] (I-2) at (0,-2.5) {$x_2$};

                    % vector input
				\node[left=-0.05cm of I-1] (b1) {};
				\node[left=-0.05cm of I-2] (b2) {};
				
				\node[hidden neuron] (H-1) at (\layersep,-1) {$z_1^{(1)}$};
				\node[hidden neuron] (H-2) at (\layersep,-2.5) {$z_2^{(1)}$};
                   				
				\node[hidden neuron] (H-4) at (2*\layersep,-1) {$\hat{z}_1^{(1)}$};
				\node[hidden neuron] (H-5) at (2*\layersep,-2.5) {$\hat{z}_2^{(1)}$};

                    \node[hidden neuron] (H-7) at (3*\layersep,-1) {$z_1^{(2)}$};
				\node[hidden neuron] (H-8) at (3*\layersep,-2.5) {$z_2^{(2)}$};
                    		
				\node[hidden neuron] (H-10) at (4*\layersep,-1) {$\hat{z}_1^{(2)}$};
				\node[hidden neuron] (H-11) at (4*\layersep,-2.5) {$\hat{z}_2^{(2)}$};

				\node[output neuron] at (5*\layersep, -1.75) (O-1) {$f(x)$};
				
				% Connect every node in the hidden layer with the output layer
				\draw[nnedge] (I-1) --node[above] {$2$} (H-1);
				\draw[nnedge] (I-1) --node[above, pos=0.3] {$-3$} (H-2);
                  
				\draw[nnedge] (I-2) --node[below, pos=0.3] {$1$} (H-1);
				\draw[nnedge] (I-2) --node[below] {$4$} (H-2);

				\draw[dashed,->] (H-1) --node[below] {ReLU} (H-4);
				\draw[dashed,->] (H-2) --node[below] {ReLU} (H-5);

				\draw[nnedge] (H-4) --node[above] {$4$} (H-7);
                    \draw[nnedge] (H-4) --node[above, pos=0.3] {$2$} (H-8);
                   
                    \draw[nnedge] (H-5) --node[below, pos=0.4] {$-2$} (H-7);
				\draw[nnedge] (H-5) --node[below, pos=0.3] {$1$} (H-8);

                    \draw[dashed,->] (H-7) --node[below] {ReLU} (H-10);
                    \draw[dashed,->] (H-8) --node[below] {ReLU} (H-11);
                    \draw[nnedge] (H-10) --node[above] {$-2$} (O-1);
				\draw[nnedge] (H-11) --node[below] {$1$} (O-1);

				% result first prop
				\node[below=0.05cm of H-1] (b1) {};
				\node[below=0.05cm of H-2] (b2) {};

                     % Biases
				\node[below=0.05cm of H-4] (b1) {};
				\node[below=0.05cm of H-5] (b2) {};

                      % Biases
				\node[below=0.05cm of H-7] (b1) {};
				\node[below=0.05cm of H-8] (b2) {};
                   
                    % Annotate the layers
		          \node[annot,left of=I-1, node distance=1cm] (hl1) {$[-2, 2]$};
            \node[annot,left of=I-2, node distance=1cm] (hl1) {$[-1, 3]$};

				\node[annot,above of=H-1, node distance=2cm] (hl1) {$\vz^{(1)}$};
            \node[annot,above of=H-1, node distance=1cm] (hl1) {$[-5, 7]$\\{\color{red}$[-4.96, 6.97]$}\\{\color{blue}$[-5.02, 6.99]$}};
            
            \node[annot,below of=H-2, node distance=1cm] {$[-10, 18]$\\{\color{red}$\hspace{-3mm}[-9.91, 17.98]$}\\{\color{blue}$\hspace{-3mm}[-10.00, 18.10]$}};

             \node[annot,below of=H-5, node distance=1cm] {$[0, 18]$\\{\color{red}$[0, 17.98]$}\\{\color{blue}$[0, 18.10]$}};
    
				\node[annot,above of=H-4, node distance=2cm] (hl2){$\hat{\vz}^{(1)}$};
                \node[annot,above of=H-4, node distance=1cm] (hl2){$[0,7]$\\{\color{red}$[0, 6.97]$}\\{\color{blue}$[0, 6.99]$}};

                    \node[annot,above of=H-7, node distance=2cm] (hl3) {$\vz^{(2)}$ };
                    \node[annot,above of=H-7, node distance=1cm] (hl3) {$[-36, 28]$\\{\color{red}$ \hspace{-3mm}[-35.96, 21.9]$}\\{\color{blue}$ \hspace{-3mm}[-36.21, 22.28]$}};
                    \node[annot,below of=H-8, node distance=1cm] (hl3) {$[0, 32]$\\{\color{red}$[0, 19.91]$}\\{\color{blue}$[0, 19.92]$}};
                     \node[annot,above of=H-10, node distance=2cm] (hl4) {$\hat{\vz}^{(2)}$ };

                     \node[annot,above of=H-10, node distance=1cm] (hl4) {$[0, 28]$\\{\color{red}$[0, 21.9]$}\\{\color{blue}$[0, 22.28]$} };

                     \node[annot,below of=H-11, node distance=1cm] (hl4) {$[0, 32]$\\{\color{red}$[0, 19.91]$}\\{\color{blue}$[0, 19.92]$} };

                    \node[annot,right of=O-1, node distance=1.5cm] (o1) {{\color{Green}$\bm{[-33, 18.9]}$}\\\hspace{-3mm}$[-56, 32]$\\{\color{red}$[-32.82, 18.8]$}\\{\color{blue}$[-34.02, 18.83]$}\\{\color{purple}$[-42, 24.3]$}};
                
                    \node[annot,above of=O-1, node distance=2.75cm] (o1) {$z^{(3)}$};

			\end{tikzpicture}
    \caption{Toy DNN used in this example. Intervals reported in green are the exact output reachable set computed via MIP, in black are the results of the IBP, and in purple the CROWN ones considering the input $[[-2,2], [-1,3]]$. In red are the reachable sets computed using a naive sampling-based approach of $n=10k$ samples. Finally, in blue, the ones computed using a naive sampling-based approach combined with the EVT error estimation.}
    \label{fig:toyDNN_new}
\end{figure}

Recalling the example provided in Sec. \ref{example_computation}, we show now the computation of the linear bounds employing CROWN enhanced with $\texttt{PT-LiRPA}$. In detail, the calculation is analogous to what we have seen above, except for the construction of the diagonal matrices and bias vectors. In the following, we will compute the linear bounds using both the theoretical results of Theorem \ref{theorem:PT_lirpa_wilks} and Theorem \ref{theorem:pt_lirpa_guarantees}.

We start by computing the estimated reachable sets from a sample-based approach in $\mathcal{C}$ using $n=10k$ samples, which for the Proposition \ref{prop:wilks_union_probability}, with $R=0.999$ and considering the number of neurons in the DNN, are sufficient to have a final confidence $\psi \geq 0.99$. We report in Figure \ref{fig:toyDNN_new} highlighted in red the estimated reachable sets obtained from the propagation of $n$ random samples drawn from $[[-2,2], [-1,3]]$. As we can notice, the bounds are slightly tighter than the overestimated ones obtained from the IBP process. Our intuition is thus that from the computation of $\underline{\mD}^{(i)},\overline{\mD}^{(i)}, \underline{\vb}^{(i)},\overline{\vb}^{(i)}$ using these tightened bounds we can obtain more accurate lower and upper final linear bounds. For the diagonal matrices and bias vectors, we get:

\begin{minipage}{0.49\linewidth}
    \begin{align*}
        &\underline{\mD}^{(2)} = \begin{bmatrix}
        \frac{u}{u-l} & 0\\
        0 & 1
    \end{bmatrix}= \begin{bmatrix}
        0.3784& 0\\
        0&1
    \end{bmatrix}\\
     &\overline{\mD}^{(2)} = \begin{bmatrix}
        \alpha & 0\\
        0 & 1
    \end{bmatrix}= \begin{bmatrix}
        0 & 0\\
        0&1
    \end{bmatrix}
    \end{align*}
\end{minipage}\hfill
\begin{minipage}{0.49\linewidth}
\begin{align*}
&\underline{\vb}^{(2)} = \begin{bmatrix}
    \frac{-ul}{u-l}\\
    0 
\end{bmatrix}= \begin{bmatrix}
    13.61\\
    0
\end{bmatrix}\\
&\overline{\vb}^{(2)} = \begin{bmatrix}
    0\\
    0 
\end{bmatrix},
\end{align*}
\end{minipage}

and 

\begin{minipage}{0.49\linewidth}
    \begin{align*}
        &\underline{\mD}^{(1)} = \begin{bmatrix}
        \frac{u}{u-l} & 0\\
        0 & \alpha
    \end{bmatrix}= \begin{bmatrix}
        0.5839 & 0\\
        0& 0
    \end{bmatrix}\\
     &\overline{\mD}^{(1)} = \begin{bmatrix}
        \frac{u}{u-l} & 0\\
        0 & \frac{u}{u-l}
    \end{bmatrix}= \begin{bmatrix}
        0.5839 & 0\\
        0& 0.6448
    \end{bmatrix}
    \end{align*}
\end{minipage}\hfill
\begin{minipage}{0.49\linewidth}
\begin{align*}
&\underline{\vb}^{(1)} = \begin{bmatrix}
    \frac{-ul}{u-l}\\
    0 
\end{bmatrix}= \begin{bmatrix}
    2.8986\\
    0
\end{bmatrix}\\
&\overline{\vb}^{(1)} = \begin{bmatrix}
    \frac{-ul}{u-l}\\
    \frac{-ul}{u-l} 
\end{bmatrix}
= \begin{bmatrix}
    2.8986\\
    6.3870 
\end{bmatrix}.
\end{align*}
\vspace{3mm}
\end{minipage}

Thus computing all the $\mA$s and $d$s vectors,

\begin{align*}
    \underline{\mA}^{(2)} &= \underline{\mA}^{(3)}\underline{\mD}^{(2)}\mW^{(2)} = [-1.0275,  2.5138],\qquad
    \overline{\mA}^{(2)} = \overline{\mA}^{(3)}\overline{\mD}^{(2)}\mW^{(2)} = [2,1],\\
    \underline{\mA}^{(1)} &= \underline{\mA}^{(2)}\underline{\mD}^{(1)}\mW^{(1)} = [-1.2, -0.6], \qquad \hspace{7.5mm}
    \overline{\mA}^{(1)} = \overline{\mA}^{(2)}\overline{\mD}^{(1)}\mW^{(1)} = [0.4013, 3.7471],\\
    \underline{d} &= \underline{\mA}^{(3)} \underline{\vb}^{(2)} + \underline{\mA}^{(2)} \underline{\vb}^{(1)} = -30.1983,\qquad \hspace{8.5mm}
    \overline{d} = \overline{\mA}^{(3)} \overline{\vb}^{(2)} + \overline{\mA}^{(2)} \overline{\vb}^{(1)} = 12.1841,
\end{align*}

we obtain:

\begin{align*}
    \underline{f}_{\text{CROWN w/ \texttt{PT-LiRPA}}} = \min\limits_{\vx \in \mathcal{C}} \underline{\mA}^{(1)} (\vx) + \underline{d} &= -\vert\vert \underline{\mA}^{(1)} \vert\vert_1 \cdot \varepsilon + \underline{\mA}^{(1)} \vx_0 + \underline{d}\\
    &= -3.6 -0.6 -30.1983 = -34.4. 
\end{align*}

\begin{align*}
    \overline{f}_{\text{CROWN w/ \texttt{PT-LiRPA}}} = \max\limits_{\vx \in \mathcal{C}} \overline{\mA}^{(1)} (\vx) + \overline{d} &= \vert\vert \underline{\mA}^{(1)} \vert\vert_1 \cdot \varepsilon + \overline{\mA}^{(1)} \vx_0 + \overline{d}\\
    &= 8.2968 + 3.7471 + 12.1841 = 24.23. 
\end{align*}

As we can notice, these bounds are significantly tighter than the ones computed using CROWN ($[-42,24.3]$). However, the theoretical guarantees provided by Theorem \ref{theorem:PT_lirpa_wilks} allow us to state that these bounds are probabilistically sound only for a fraction $R$ of the perturbation region $\mathcal{C}$, thus resulting in a slightly weaker guarantee w.r.t. the one provided by existing probabilistic approaches. Nonetheless, we believe that if one accepts the assumption underlying this theoretical guarantee, this approach still presents a valuable and computationally efficient tool for computing probabilistically valid linear output bounds. In the following, we show how to practically extend these theoretical guarantees to the whole perturbation region, exploiting Theorem \ref{theorem:pt_lirpa_guarantees}.

We start again from computing the estimated reachable sets from a sample-based approach in $\mathcal{C}$. For each estimated lower and upper bound, we compute and add the corresponding error using $\frac{Y_2 - Y_1}{(1 - p)^{-a_l} - 1}$ for the lower and $\frac{Y_n - Y_{n-1}}{(1 - p)^{-a_u} - 1}$ for the upper bound, respectively. We report in Figure \ref{fig:toyDNN_new} highlighted in blue the new estimated reachable sets obtained from the propagation of $n$ random samples drawn from $[[-2,2], [-1,3]]$ with the addition of the corresponding error. Hence, we speculate that recomputing $\underline{\mD}^{(i)},\overline{\mD}^{(i)}, \underline{\vb}^{(i)},\overline{\vb}^{(i)}$ using these new estimated reachable sets we can still obtain more accurate lower and upper final linear bounds w.r.t. the LiRPA-based approaches. In fact, we obtain:

\begin{minipage}{0.49\linewidth}
    \begin{align*}
        &\underline{\mD}^{(2)} = \begin{bmatrix}
        \frac{u}{u-l} & 0\\
        0 & 1
    \end{bmatrix}= \begin{bmatrix}
        0.3809& 0\\
        0&1
    \end{bmatrix}\\
     &\overline{\mD}^{(2)} = \begin{bmatrix}
        \alpha & 0\\
        0 & 1
    \end{bmatrix}= \begin{bmatrix}
        0 & 0\\
        0&1
    \end{bmatrix}
    \end{align*}
\end{minipage}\hfill
\begin{minipage}{0.49\linewidth}
\begin{align*}
&\underline{\vb}^{(2)} = \begin{bmatrix}
    \frac{-ul}{u-l}\\
    0 
\end{bmatrix}= \begin{bmatrix}
    13.7919\\
    0
\end{bmatrix}\\
&\overline{\vb}^{(2)} = \begin{bmatrix}
    0\\
    0 
\end{bmatrix},
\end{align*}
\end{minipage}

and 

\begin{minipage}{0.49\linewidth}
    \begin{align*}
        &\underline{\mD}^{(1)} = \begin{bmatrix}
        \frac{u}{u-l} & 0\\
        0 & \alpha
    \end{bmatrix}= \begin{bmatrix}
        0.5820 & 0\\
        0& 0
    \end{bmatrix}\\
     &\overline{\mD}^{(1)} = \begin{bmatrix}
        \frac{u}{u-l} & 0\\
        0 & \frac{u}{u-l}
    \end{bmatrix}= \begin{bmatrix}
        0.5820 & 0\\
        0& 0.6441
    \end{bmatrix}
    \end{align*}
\end{minipage}\hfill
\begin{minipage}{0.49\linewidth}
\begin{align*}
&\underline{\vb}^{(1)} = \begin{bmatrix}
    \frac{-ul}{u-l}\\
    0 
\end{bmatrix}= \begin{bmatrix}
    2.9219\\
    0
\end{bmatrix}\\
&\overline{\vb}^{(1)} = \begin{bmatrix}
    \frac{-ul}{u-l}\\
    \frac{-ul}{u-l} 
\end{bmatrix}
= \begin{bmatrix}
    2.9219\\
    6.4427 
\end{bmatrix}.
\end{align*}
\end{minipage}

We can now compute all the $\mA$s and $d$s vectors.

\begin{align*}
    \underline{\mA}^{(2)} &= \underline{\mA}^{(3)}\underline{\mD}^{(2)}\mW^{(2)} = [-1.0471,  2.5235]\\
    \overline{\mA}^{(2)} &= \overline{\mA}^{(3)}\overline{\mD}^{(2)}\mW^{(2)} = [2,1]\\
    \underline{\mA}^{(1)} &= \underline{\mA}^{(2)}\underline{\mD}^{(1)}\mW^{(1)} = [-1.2187, -0.6094]\\
    \overline{\mA}^{(1)} &= \overline{\mA}^{(2)}\overline{\mD}^{(1)}\mW^{(1)} = [0.3957, 3.7402]\\
    \underline{d} &= \underline{\mA}^{(3)} \underline{\vb}^{(2)} + \underline{\mA}^{(2)} \underline{\vb}^{(1)} = -30.6434\\
    \overline{d} &= \overline{\mA}^{(3)} \overline{\vb}^{(2)} + \overline{\mA}^{(2)} \overline{\vb}^{(1)} = 12.2865
\end{align*}

Finally we have

\begin{align*}
    \underline{f}_{\text{PT-LiRPA}} = \min\limits_{\vx \in \mathcal{C}} \underline{\mA}^{(1)} (\vx) + \underline{d} &= -\vert\vert \underline{\mA}^{(1)} \vert\vert_1 \cdot \varepsilon + \underline{\mA}^{(1)} \vx_0 + \underline{d}\\
    &= -3.6562 -0.6094 -30.6434 = -34.91. 
\end{align*}

\begin{align*}
    \overline{f}_{\text{PT-LiRPA}} = \max\limits_{\vx \in \mathcal{C}} \overline{\mA}^{(1)} (\vx) + \overline{d} &= \vert\vert \underline{\mA}^{(1)} \vert\vert_1 \cdot \varepsilon + \overline{\mA}^{(1)} \vx_0 + \overline{d}\\
    &= 8.2717 + 3.7402 + 12.2865 = 24.3. 
\end{align*}

Although the upper bound is equivalent to the original CROWN approach, we can notice that our procedure produces a tighter lower bound. This toy example provides a preliminary insight into the potential of the proposed solution. Our speculation on the impact of \texttt{PT-LiRPA} on realistic verification instances will be confirmed by the experiments presented in Sec. \ref{sec:empirical}.

\paragraph{EVT-based approach to directly bound the output?}

A natural question that arises is whether the results of Theorem \ref{theorem:pt_lirpa_guarantees}
can be used directly to obtain a tight estimation of the output reachable set, without relying on the LiRPA combination.
Although this sampling-based method offers a probabilistic estimate that, with high confidence, contains the entire perturbation region, it may still underestimate the true output bounds due to its reliance on a finite number of samples. For example, the MIP result yields bounds of $[-33.0, 18.9]$, and as we can notice in Fig. \ref{fig:toyDNN_new} highlighted in blues, the output reachable set only applying a forward computation of Theorem \ref{theorem:pt_lirpa_guarantees} still underestimates the exact upper bound $[-34.02, 18.83]$. In contrast, since our method integrates a sampling-based approach with any LiRPA method—which inherently provides sound overestimations—the final computed bounds will always be at least as tight as those obtained through estimation based on a finite number of samples and are likely (with a confidence at least $1-2mp$) to produce valid linear bounds, i.e., not discarded by potential adversarial attacks (in fact, we obtain as final result $[-34.4, 24.23]$). Additionally, as emphasized in prior work \citep{autolirpa}, combining forward and backward analysis typically yields tighter bounds compared to using a simple forward bound computation. This observation further motivated our investigation into how tighter reachable sets can enhance the linearization approaches for verification efficiency.

\section{\texttt{PT-LiRPA} Framework for Neural Network Verification}\label{sec:pt_lirpa_algo}

Based on the theoretical results of Sec. \ref{sec:method}, we now present in Algorithm \ref{alg:pt_lirpa} the \texttt{PT-LiRPA} approach for the verification process. For the sake of clarity and without loss of generality, we present the procedure applied to the parallel BaB as shown for the optimized LiRPA approach proposed in \citep{acrown}.
Given a DNN $f$ with $m$ neurons and a region of interest $\mathcal{C}$, the verification process typically involves a projected gradient descent (PGD) attack \citep{pgd}. 
\begin{algorithm}[b]
\caption{\texttt{PGD\_Attack}\cite{pgd}}\label{alg:pgd_linf}
\begin{algorithmic}[1]
\small
\STATE \textbf{Input} Original input $\vx_0$, neural network $f$, maximum perturbation $\varepsilon$ to create data range $[x_{\min},x_{\max}]$, step size $\alpha$, iterations $T$,  \textit{random\_start}
\STATE \textbf{Output} adversarial example $\vx_{\text{adv}}$ with $\|\vx_{\text{adv}}-\vx_0\|_\infty \le \varepsilon$ or original input $\vx_0$
\vspace{0.6em}

\IF{\textit{random\_start}}
    \STATE $\vx \leftarrow \vx_0 + \operatorname{Uniform}(-\varepsilon,\varepsilon)$ 
\ELSE
    \STATE $\vx \leftarrow \vx_0$
\ENDIF
\STATE $\vx_{\text{adv}} \leftarrow \operatorname{clip}(\vx, x_{\min}, x_{\max})$
\FOR{$t \in \{1,\ldots, T\}$}
    \IF{$f(\vx_{\text{adv}}) \le 0$}
        \STATE \textbf{return} $\vx_{\text{adv}}$
    \ENDIF

    \STATE $g \leftarrow \nabla_{\vx_{\text{adv}}} \,f(\vx_{\text{adv}})$ \hspace*{\fill} $\rhd$  gradient of scalar output w.r.t. adversarial input

    \STATE $\vx_{\text{adv}} \leftarrow \vx_{\text{adv}} - \alpha \cdot \operatorname{sign}(g)$ \hspace*{\fill} $\rhd$  descent step for minimizing $s$
    
    \STATE $\Delta \leftarrow \operatorname{clip}(\vx_{\text{adv}} - \vx_0, -\varepsilon, \varepsilon)$ \hspace*{\fill} $\rhd$  project perturbation onto $L_\infty$ ball
    
    \STATE $\vx_{\text{adv}} \leftarrow \operatorname{clip}(\vx_0 + \Delta, x_{\min}, x_{\max})$
\ENDFOR

\RETURN $\vx_0$
\end{algorithmic}
\end{algorithm}
This attack, reported in Alg. \ref{alg:pgd_linf} for the sake of completeness, can be employed before, after, or during the BaB procedure to search for potential adversarial inputs within the input region under consideration.  
\begin{algorithm}[t]
\caption{\texttt{PT-LiRPA} on parallel BaB}\label{alg:pt_lirpa}
\begin{algorithmic}[1]
\small
\STATE \textbf{Input:} A DNN $f$ with $N$ layers and $m$ neurons, an original input $\vx_0$, a maximum $\varepsilon$ perturbation to create a perturbation region $\mathcal{C}$, maximum error in the confidence $p$, sample size $n$, $\xi \in (0,1)$ for $\nu(n)$ and a batch size $t$.
\STATE \textbf{Output:} \textit{robust}/\textit{not-robust}
\vspace{0.2cm}

\hspace*{\fill} $\rhd$ as in Alg. \ref{alg:pgd_linf}
\IF{\texttt{PGD\_attack}($f, \mathcal{C}$)} 
    \STATE \textbf{return} \textit{not robust} 
\ENDIF

\STATE{\textit{interm\_bounds} $\gets \texttt{get\_interm\_bounds}(f, \mathcal{C}, n, p, \xi)$} \hspace*{\fill} $\rhd$ as in Alg. \ref{alg:interm_bounds}
\STATE{$\underline{f}_{\mathcal{C}}, \overline{f}_{\mathcal{C}} \gets \texttt{LiRPA}(f, \mathcal{C}, \textit{interm\_bounds}$)} \hspace*{\fill} $\rhd$ as in Alg. \ref{alg:lirpa} where $\mathcal{C}$ contains $\vx_0, \varepsilon$

\STATE{$\mathcal{B} \gets (\underline{f}_{\mathcal{C}}, \overline{f}_{\mathcal{C}})$}

\WHILE{$\mathcal{B} \neq \emptyset$}
{
      \STATE{$\mathcal{C}_1, \dots, \mathcal{C}_t \gets \texttt{split}(\mathcal{B}, t)$}
      \STATE{$\textit{interm\_bounds}_{\mathcal{C}_1,\dots,\mathcal{C}_t}
      \gets\texttt{get\_interm\_bounds}(f, [\mathcal{C}_1,\dots,\mathcal{C}_t], n, p, \xi)$}
      \STATE{$(\underline{f}_{\mathcal{C}_{1}}, \overline{f}_{\mathcal{C}_{1}}),\ldots,(\underline{f}_{\mathcal{C}_{t}}, \overline{f}_{\mathcal{C}_{t}})  \gets \texttt{LiRPA}(f, [\mathcal{C}_1,\ldots, \mathcal{C}_t],$\textit{interm\_bounds}$_{\mathcal{C}_1,\ldots,\mathcal{C}_t})$}
      \hspace*{\fill} $\rhd$ parallel exec. of Alg. \ref{alg:lirpa} on $\mathcal{C}_{1,\ldots,t}$
      
      \STATE $\mathcal{B}' \gets (\underline{f}_{\mathcal{C}_1}, \overline{f}_{\mathcal{C}_1}), \dots, (\underline{f}_{\mathcal{C}_t}, \overline{f}_{\mathcal{C}_t})$

      \IF{$\exists\;\mathcal{C}_i \in \mathcal{B'}\; s.t.\; \overline{f}_{\mathcal{C}_i} < 0$ or \texttt{PGD\_attack}$(f,C_i)$}
        {
        \STATE{\textbf{return} \textit{not robust}}
        }
        \ENDIF}
      \STATE{$\mathcal{B} \gets \mathcal{B}' \setminus \texttt{get\_robust\_domains}(\mathcal{B}')$}  
\ENDWHILE

\STATE{\textbf{return} \textit{robust}}
\end{algorithmic}

\end{algorithm}

In detail, we report in our \texttt{PT-LiRPA} algorithm (Alg. \ref{alg:pt_lirpa}), a potential employment of PGD during the verification process. Specifically, the attack is performed before and during the BaB, performing a projected gradient descent search in the $L_{\infty}$ ball $\mathcal{C} = \{\vx:\|\vx-\vx_0\|_{\infty}\leq\varepsilon\}$ to find an adversarial input $\vx_{\text{adv}}$ that makes the scalar model output non-positive, i.e., $f(\vx_{\text{adv}})\leq 0$. Starting from either the clean original input or a random uniform perturbation in $[-\varepsilon,\varepsilon]$ (i.e., a random input vector in the $\mathcal{C}$) the method iteratively evaluates the scalar output, computes its gradient with respect to the input, and takes an $L_{\infty}$-constrained descent step using the elementwise sign of the gradient. After each update, the perturbation is projected back onto the $L_{\infty}$ ball and the input is clipped to the valid data range; the procedure stops early if a negative output is obtained and otherwise runs for at most $T$ iterations.\\ The main hyperparameters are the maximum perturbation $\varepsilon$, the step size $\alpha$ (typically chosen on the order of $\varepsilon/T$), the maximum iterations $T$, and the optional random start; multiple restarts or momentum can be used to increase attack strength. Success provides a concrete counterexample to robustness within the prescribed $L_{\infty}$ radius, while failure is only a heuristic indication and does not constitute a formal certificate of robustness.\\
Hence, Alg. \ref{alg:pt_lirpa} begins with a PGD attack (lines 3-5), and if no adversarial is found, we proceed with the Branch-and-Bound process. 
We compute the estimated reachable sets using the \texttt{get\_interm\_bounds} method (line 6), which exploits the results of Theorem \ref{theorem:pt_lirpa_guarantees} and is reported here below in Alg. \ref{alg:interm_bounds} for clarity. 

\begin{algorithm}[h!]
\caption{\texttt{get\_interm\_bounds}}\label{alg:interm_bounds}

\begin{algorithmic}[1]
\small

\STATE \textit{interm\_bounds} $\gets \{\}$
\STATE  $\vx_1,\ldots,\vx_n\gets \operatorname{UniformSampling}(\mathcal{C}, n)$ \hspace*{\fill} $\rhd$  collect $n$ random i.i.d inputs from $\mathcal{C}$
\FOR{each intermediate layer}
\STATE $\hat{\vl}, \hat{\vu} \gets \{\}$
\FOR{each node $z$ in layer nodes}

    \STATE $Y_1,\ldots, Y_n$ $\gets \operatorname{Sort}(z(\vx_1), \ldots, z(\vx_n))$ \hspace*{\fill} $\rhd$  with $z(\cdot)$ as in Theorem \ref{theorem:pt_lirpa_guarantees}
    \STATE $\overline{l}, \underline{u} \gets$ $Y_1$, $Y_n$
    \STATE $a_l, a_u \gets \frac{\log(\nu)}{\log\left( \frac{Y_\nu - Y_3}{Y_3 - Y_2} \right)}, \frac{\log(\nu)}{\log\left( \frac{Y_{n-2} - Y_{n-\nu}}{Y_{n-1} - Y_{n-2}} \right)}$
    \STATE $\hat{l},\hat{u} \gets \overline{l} - \frac{Y_2 - Y_1}{(1 - p)^{-a_l} - 1}, \underline{u} + \frac{Y_n - Y_{n-1}}{(1 - p)^{-a_u} - 1}$\hspace*{\fill} $\rhd$  as in Eq. \ref{gaps}
    \STATE $\hat{\vl},\hat{\vu}  \gets \hat{\vl} \cup \hat{l}, \hat{\vu} \cup \hat{u}$

\ENDFOR
\STATE \textit{interm\_bounds} $\gets$ \textit{interm\_bounds} $\cup [\hat{\vl}, \hat{\vu}]$ \hspace*{\fill} $\rhd$ store the vector of lower and upper bounds for the specific layer
\ENDFOR
\STATE{\textbf{return} \textit{interm\_bounds}}
\end{algorithmic}
\end{algorithm}
We then use these bounds in the linear bounds computation on any existing \texttt{LiRPA} approach (line 6), following Alg. \ref{alg:lirpa} of Sec. \ref{sec:lirpa_preliminaries} and the computation shown in the toy example of Sec. \ref{subsec:example_computation}. We store the resulting output bounds $\underline{f}$ and $\overline{f}$ for the region $\mathcal{C}$, namely $\underline{f}_{\mathcal{C}}$ and $\overline{f}_{\mathcal{C}}$ in a set $\mathcal{B}$ of unverified regions (line 7).
\begin{algorithm}[b]
\caption{\texttt{get\_robust\_domains}}\label{alg:get_robust_domains}
\begin{algorithmic}[1]
\small
\STATE \textit{robust\_domains} $\gets \{\}$
\FOR{$(\underline{f}_{\mathcal{C}_i}, \overline{f}_{\mathcal{C}_i}) \in \mathcal{B}$}
\IF{$\underline{f}_{\mathcal{C}_i} >0$}
    \STATE \textit{robust\_domains} $\gets$\textit{robust\_domains} $\cup \;\mathcal{C}_i$ \hspace*{\fill} $\rhd$ following Def. \ref{def:decision_problem_new}
\ENDIF
\ENDFOR
\STATE{\textbf{return} \textit{robust\_domains}}
\end{algorithmic}
\end{algorithm}
We then continue the BaB process by splitting (using the \texttt{split} method) the original region from $\mathcal{B}$ into $t$ sub-regions (line 9). Notably, we can perform the parallel selection and splitting into sub-domains using information on unstable ReLU nodes, as shown in \citep{babSplitRelu,bcrown}, or just on the perturbation region $\mathcal{C}_i$ \citep{reluval}. Once we have the new sub-domains, we recompute the estimated reachable sets in parallel and use these bounds for the new computation of the linear lower and upper bounds for each sub-region, and we update $\mathcal{B}$ with the resulting unverified sub-domains (lines 10-12). At each iteration, the process can end either because there is at least a single sub-domain $\mathcal{C}_i \in \mathcal{B}$ that presents $\overline{f} < 0$, or a PGD attack succeeds, thus returning \textit{not robust} as the answer (lines 13-16). Otherwise, the process continues updating $\mathcal{B}$ with the unverified domains
using the procedure \texttt{get\_robust\_domains} (line 15) reported in Alg. \ref{alg:get_robust_domains}. Following Theorem \ref{theorem:pt_lirpa_guarantees}, if we reach the emptiness of $\mathcal{B}$, thus no adversarial examples are found during the verification process and all the sub-domains are evaluated as \textit{robust}, we can state that, with a confidence $\geq 1-2mp$, the DNN is robust for the whole perturbation region $\mathcal{C}$.

\section{Empirical Evaluation}\label{sec:empirical}
Our empirical evaluation consists of three main experiments to answer the following questions:
\begin{description}
    \item [\textbf{Q1.}] \textit{How does the hyperparameter $\xi$ impact the lower bound computation? How does the number of samples employed in the computation process impact the tightening process? }

    \item [\textbf{Q2.}] \textit{How much \texttt{PT-LiRPA} improves the robustness bounds certificates w.r.t. other probabilistic and worst-case methods? What is the computational overhead of the proposed solution with respect to a worst-case certification approach?}
    
    \item [\textbf{Q3.}] \textit{What is the general impact of \texttt{PT-LiRPA} in the verification process of challenging instances such as the one employed in the VNN-COMP \citep{VNN-comp2022,VNN-comp2023}?}
\end{description}

All data are collected on a cluster running Rocky Linux 9.34 equipped with Nvidia RTX A6000 (48 GiB) and a CPU AMD Epyc 7313 (16 cores). The code, trained models, and comprehensive instructions for reproducing our results are available at \url{https://github.com/lmarza/ProbVerNet}.
\paragraph{\textbf{Answers to Q1.}} To address the first question, we consider the original pre-trained models on MNIST dataset. Specifically, we focus on MLP models with varying depths and activation functions. For consistency and ease of comparison, we adopt the same notation as \citep{crown} and \citep{proven}: a model is denoted by the dataset name, followed by the number of layers $i$, the number of neurons per layer $j$, formatted as $i\times[j]$, and the activation function used.
We begin by analyzing the mean percentage error in estimating the lower bounds of the intermediate reachable sets (using Eq.\;\ref{eq:error}) for a fixed input image, using \texttt{PT-LiRPA} on various neural networks trained on the MNIST dataset.

\begin{equation}\label{eq:error}
    error = \frac{Y_2-Y_1}{(1-p)^{-a} - 1}
\end{equation}

with $a \approx  \frac{\log(\nu)}{\log\left( \frac{Y_\nu - Y_3}{Y_3 - Y_2} \right)}$, where $\nu = \lfloor n^\xi \rfloor$, $\xi \in (0,1)$.

\begin{figure}[b]
    \centering    \includegraphics[width=0.32\textwidth]{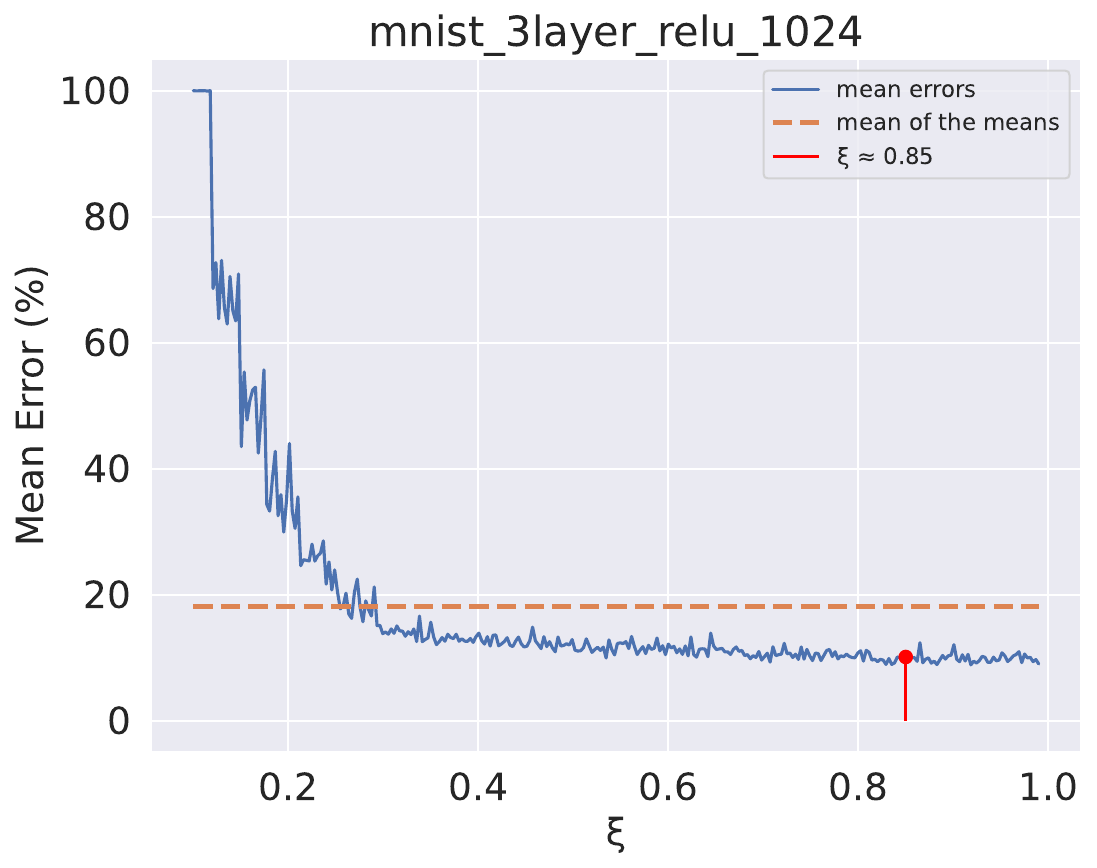}
    \includegraphics[width=0.32\textwidth]{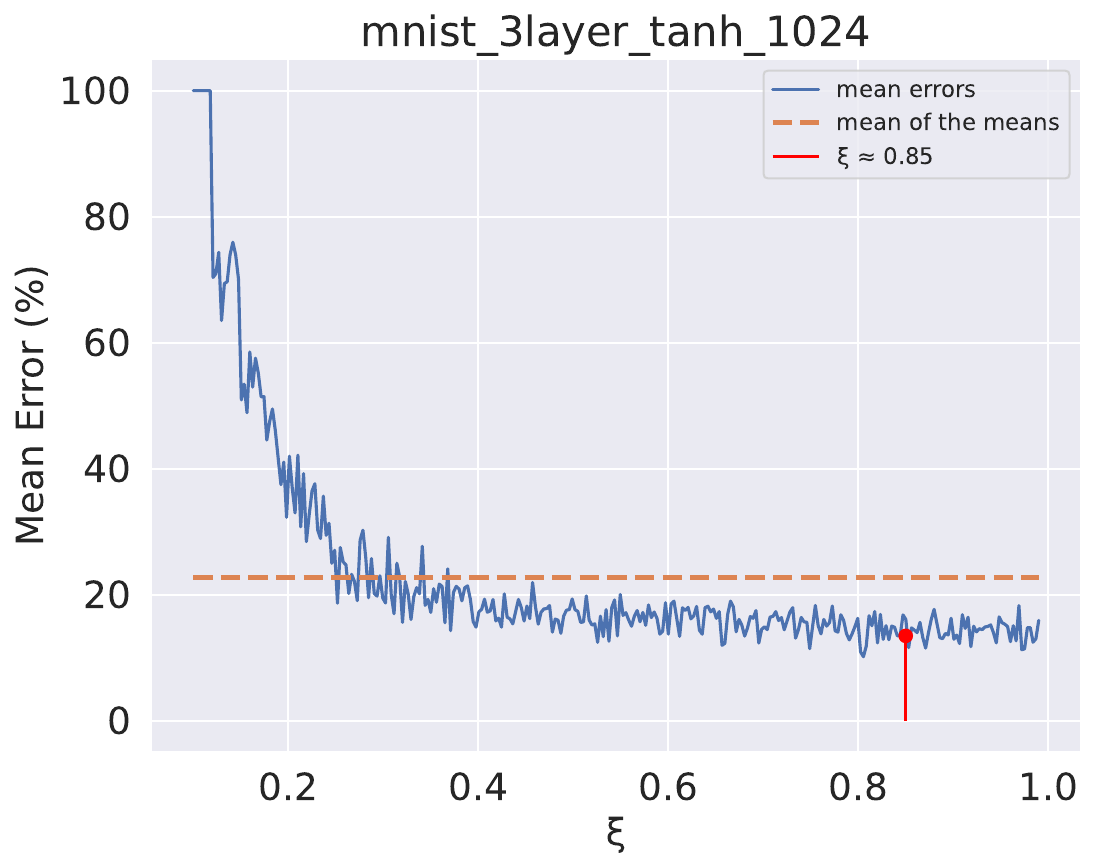}
    \includegraphics[width=0.32\textwidth]{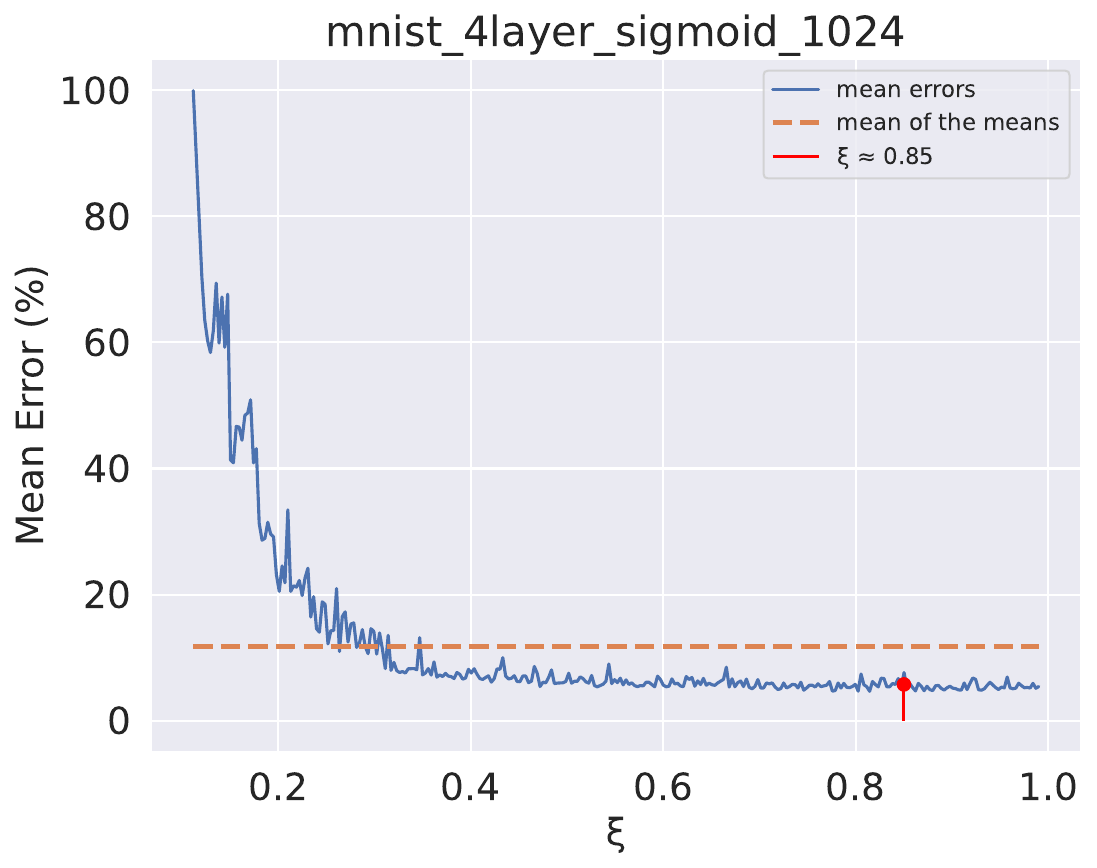}
    \caption{Mean distance error (\%) achieved on each intermediate node using PT-LiRPA with EVT-based error computation, for networks with ReLU (left), Tanh (middle), and Sigmoid (right) activations. In red we report the value  $\xi=0.85$ selected for the following experiments.}
    \label{fig:evterror}
\end{figure}

In particular, we first study the impact of the $\xi \in (0,1)$ hyperparameter, which controls the number of order statistics $\nu = \lfloor n^\xi \rfloor$ used to estimate the tail distribution and compute the mean distance error across all intermediate nodes, and thus the neural network. Our results reported in Fig. \ref{fig:evterror} show that, using a fixed sample size of $n=10k$, for small values of $\xi$, the number of extreme samples is limited, leading to high variance and bias in the tail modeling, and consequently to considerable estimation errors. As $\xi$ increases in the interval $(0.2, 0.6)$, the error decreases significantly due to the improved reliability of the extreme value statistics. For larger $\xi$ values, around $[0.6,1)$, the error stabilizes below the overall mean error across all different $\xi$ values, indicating that the estimator becomes robust and further improvements are marginal. We highlight that, even without access to the true lower (or upper, respectively) bound, one can in principle fine-tune the $\xi$ parameter as desired to minimize the error with respect to the target value for a given application.  Based on these observations, we set $\xi = 0.85$ for the following experiments, as it provides a good trade-off between low error estimation and stability across different models.

Hence, for the fixed value $\xi$ and perturbed input, we investigate the impact of three sample sizes, namely $10000, 100000,$ and $350000$, in the lower bound estimation of each estimated reachable set. For each sample size, we compute the mean distance error across the entire network setting $p=0.01$, thus ensuring a confidence level of at least 99\% in the final results.  
For the distance error achieved on each intermediate node of the network using Eq.\;\ref{eq:error} with $\xi=0.85$. We then compute the percentage error relative to the smallest observed value in the sample. Specifically, for each neuron, the error is normalized by the absolute value of the smallest pre-activation observed value.

\begin{table}[h]
    \centering
     \caption{Mean maximum error in estimating the lower bound of the intermediate reachable set for a fixed input image, using \texttt{PT-LiRPA} with $1-p=0.99$ on various neural networks trained on the MNIST dataset with different sample sizes.}
     \resizebox{\textwidth}{!}{
    \begin{tabular}{l||cc|cc|cc} 
    \hline 
    \textbf{Name model} & \textbf{\# samples} & \textbf{mean distance} & \textbf{\# samples} & \textbf{mean distance} & \textbf{\# samples} & \textbf{mean distance} \\ 
    \multicolumn{1}{c||}{} & & \textbf{error (\%)} &  & \textbf{error (\%)} & & \textbf{error (\%)}\\
    % \hline
    % MNIST 3$\times[20]$, ReLU & 5000 & 0.28 & 10000 & 0.18 & 50000 & 0.17\\ 
    \hline
    MNIST 2$\times[1024]$, ReLU & 10000 & 12.91 & 100000 & 9.35 & 350000 & 8.78 \\ 
    \hline
    MNIST 3$\times[1024]$, ReLU & 10000 & 9.63 & 100000 & 8.49 &350000 & 7.09\\ 
    \hline
    MNIST 4$\times[1024]$, ReLU & 10000 & 11.89 & 100000 & 9.34 &350000 & 8.85\\    
    \hline
    MNIST 2$\times[1024]$, Tanh & 10000 & 11.31 & 100000 & 12.22 & 350000& 10.49\\ 
    \hline
    MNIST 3$\times[1024]$, Tanh & 10000 & 10.67 & 100000 & 11.01 & 350000 & 10.34\\ 
    \hline
    MNIST 4$\times[1024]$, Sigmoid & 10000 & 5.03 & 100000 & 6.21 & 350000 & 4.36\\ 
    \hline
     & \textbf{mean error} & 10.24\% & \textbf{mean error} & 9.44\%  & \textbf{mean error} & 8.32\%  \\ 
    \end{tabular}}
   
    \label{tab:bound}
\end{table}

The results of Tab.\;\ref{tab:bound} demonstrate that with high confidence (i.e., at least 99\%) across all tested networks, increasing the sample size, the percentage error in the estimation decreases. Importantly, even with a limited sample size (e.g., $10000$ samples), we note that the mean error across all the networks in the lower bound estimation is $10.24\%$ from the estimated one, while increasing the sample size, we reach a maximum error of $8.32\%$.  

To assess the practical impact of the estimation error on the final output bounds, we consider the same perturbed input of the previous experiments and fix the sample size to $350k$. We then compare the final output bounds obtained by an exact MIP-based verification \citep{MIP}, the state-of-the-art $\alpha$-CROWN \citep{acrown} method, and $\alpha$-CROWN \citep{acrown} enhanced with our \texttt{PT-LiRPA}. As a representative case, we select the \textit{MNIST\_2x[1024]\_ReLU} network. 
\begin{figure}[b]
    \centering
    \includegraphics[width=0.9\linewidth]{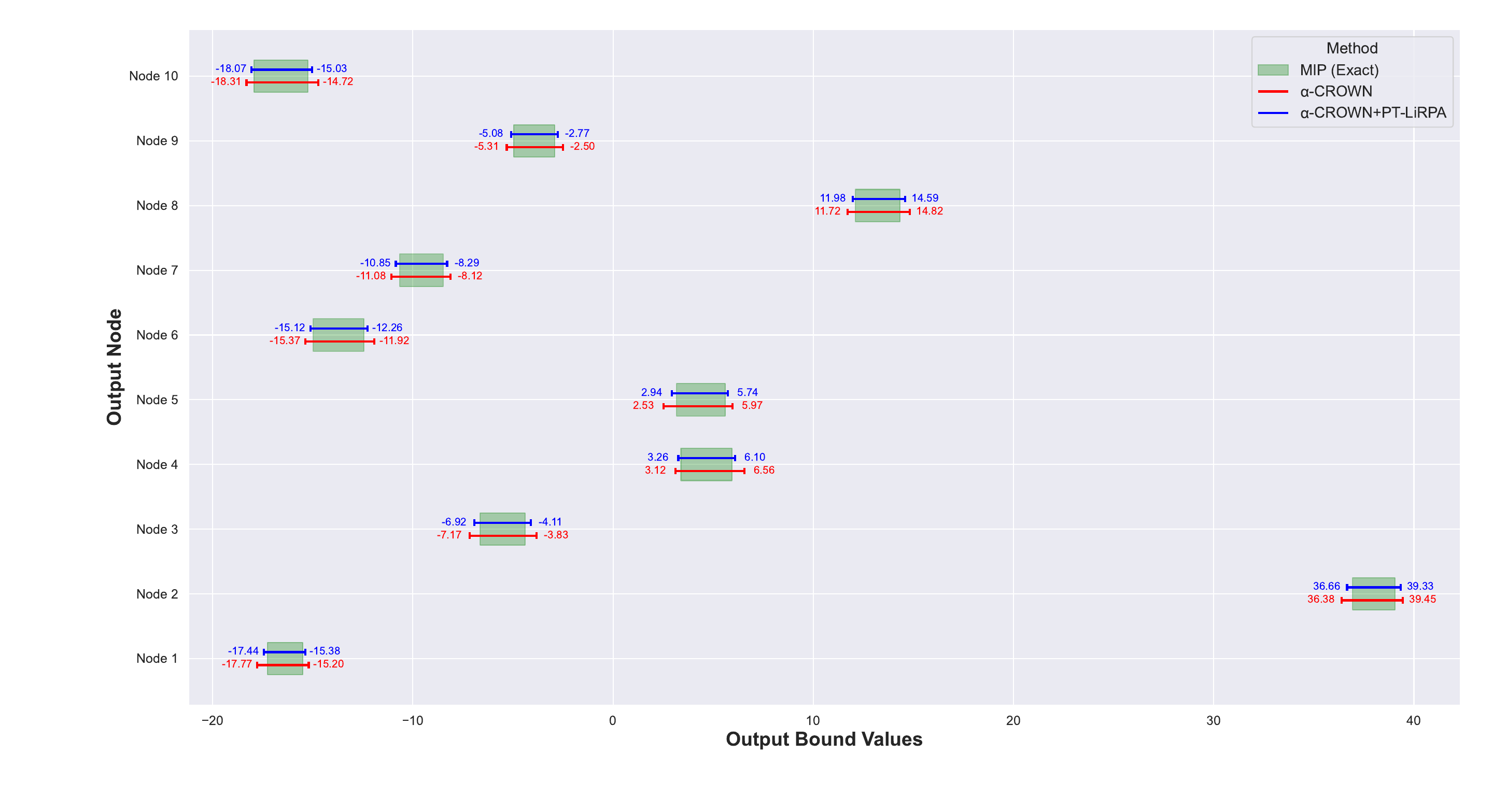}
    \vspace{-1cm}
    \caption{Comparison of output bounds on the \textit{MNIST\_2x[1024]\_ReLU} network using  $\alpha$-CROWN \citep{acrown} reported in red,$\alpha$-CROWN \citep{acrown} with \texttt{PT-LiRPA} in blue, and exact MIP verification \citep{MIP} in green.}
    \label{fig:bound_comparison}
\end{figure}
The reason for selecting this network is to enable a comparison with MIP \cite{MIP}, which provides exact output bounds. Since MIP solvers are inherently designed for ReLU, and more generally, piecewise-linear activations, they are not directly applicable to networks with Sigmoid or Tanh activations. Among the models listed in Table \ref{tab:bound}, the ones with larger estimation errors under ReLU activations with confidence $1-p=0.99$ are \textit{MNIST\_2x[1024]\_ReLU} and \textit{MNIST\_4x[1024]\_ReLU}. We select the former because MIP scalability becomes a limiting factor on larger architectures, making \textit{MNIST\_2x[1024]\_ReLU} the most suitable candidate for this analysis.
The results, reported in Fig.\;\ref{fig:bound_comparison}, show that despite the estimation error, \texttt{PT-LiRPA} produces tighter output bounds than $\alpha$-CROWN, while soundly overapproximating the exact MIP bounds. Similar trends, not reported here for the sake of readability, were consistently observed across all evaluated networks compared with $\alpha$-CROWN.

We also perform an additional experiment to analyze the impact of the confidence level on the tightness of the bounds. Specifically, we consider the same model of Fig. \ref{fig:bound_comparison}, i.e., \textit{MNIST\_2x[1024]\_ReLU}, where we have the possibility of employing the exact MIP solver and evaluated a range of increasing confidence levels, namely
$1-p \in \{0.8, 0.9, 0.95, 0.99, 0.995, 0.996, 0.997, 0.998, 0.999\}.$ For each confidence level, we compute the bounds using $\alpha$-CROWN, $\alpha$-CROWN enhanced with our \texttt{PT-LiRPA}, and MIP (which provides the exact bounds). To measure the tightness, we calculate for each of the 10 output nodes the distance between the two bounds computed. For example, if MIP returns output bounds $[-2.2, 3.5]$ and \texttt{PT-LiRPA} returns $[-2.4, 3.7]$, the distance is given by the sum of the absolute differences between the lower and upper bounds, i.e., $ |-2.2 +2.4| + |3.7-3.5| = 0.4$. This procedure is repeated for all 10 nodes, collecting, for each confidence level, the mean and standard deviation of these distances.

\begin{figure}[h!]
    \centering
    \includegraphics[width=0.75\linewidth]{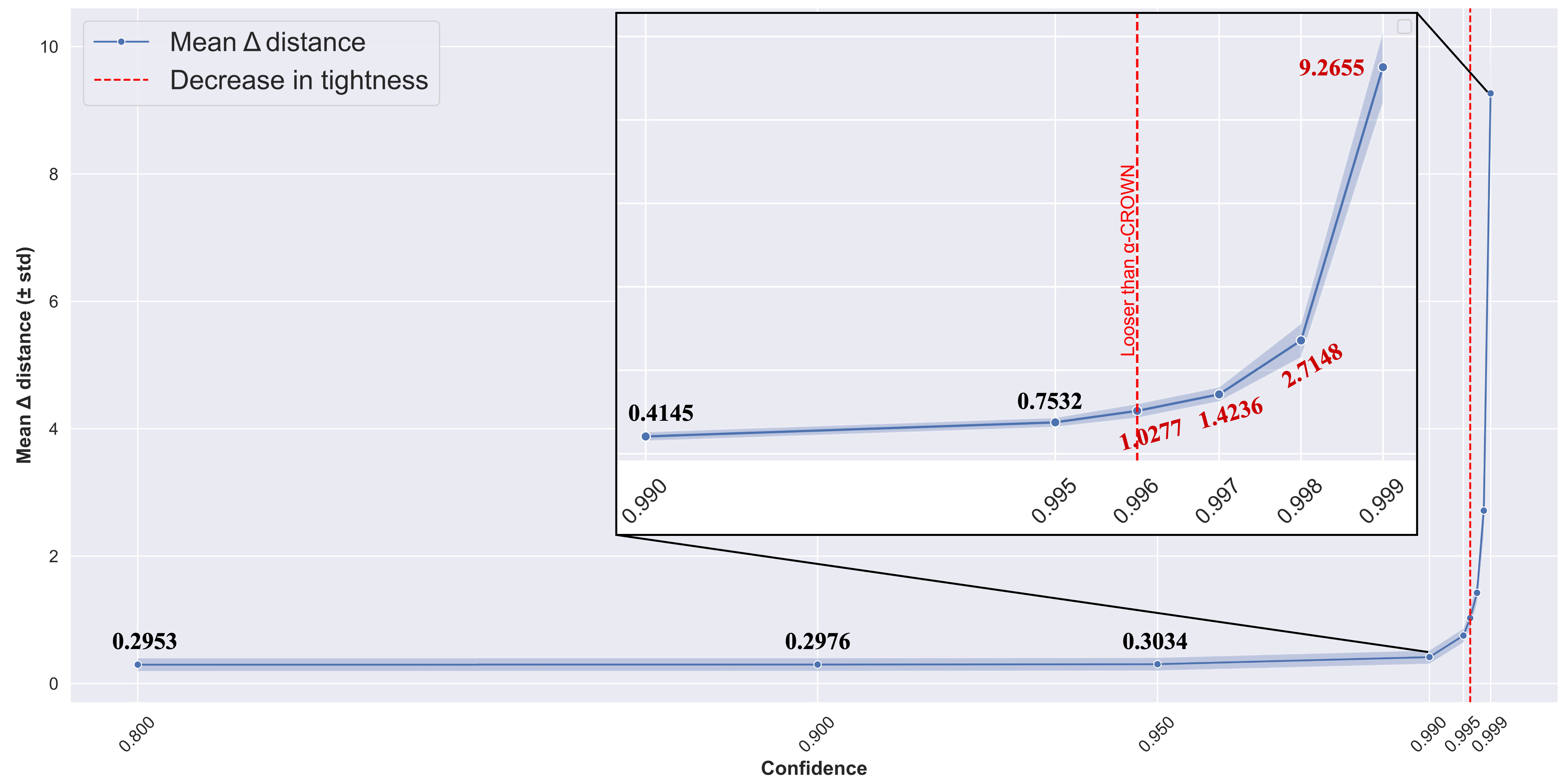}
    \caption{Mean tightness level of \texttt{PT-LiRPA} for increasing confidence level on \textit{MNIST\_2x[1024]\_ReLU} model.}
    \label{fig:exp_confidence}
\end{figure}

The results in Fig. \ref{fig:exp_confidence} clearly show that for moderate confidence levels, i.e., $1-p \in [0.8, 0.995]$, \texttt{PT-LiRPA} consistently produces sound and tighter bounds compared to $\alpha$-CROWN. As the confidence level approaches 1 (i.e., as $p \to 0$), the distance between \texttt{PT-LiRPA} and the exact bounds increases, leading to looser bounds than $\alpha$-CROWN. This trend is perfectly in line with the probabilistic nature of our approach. In fact, considering, for instance, the formula used to compute the probabilistic lower bound, $\hat{l} = Y_1 - \frac{Y_2 - Y_1}{(1 - p)^{-a_l} - 1}$ where $a_l \approx \frac{\log(\nu)}{\log\left( \frac{Y_\nu - Y_3}{Y_3 - Y_2} \right)}$ ($\nu = \lfloor n^\xi \rfloor = 350k^{0.85} \simeq 50k$), we observe that as $p \to 0$ (corresponding to requiring very high confidence), the denominator $(1-p)^{-a_l} - 1$ tends to 0 faster than the numerator $(Y_2 - Y_1)$. As a result, the error term $\frac{Y_2 - Y_1}{(1 - p)^{-a_l} - 1}$
becomes large, and a correspondingly larger margin must be subtracted from the observed minimum. Consequently, the bounds become looser as the confidence level approaches $1$. This behavior is mathematically unavoidable since we cannot guarantee arbitrarily high confidence levels while simultaneously maintaining very tight bounds. The plot therefore not only confirms the soundness of our method but also highlights the expected trade-off between confidence and tightness, which is especially relevant for safety-critical applications. For the sake of completeness, we also tested additional models such as \textit{MNIST\_2x[1024]\_Tanh} and \textit{MNIST\_4x[1024]\_Sigmoid}. Although MIP was not applicable in these cases, we compared \texttt{PT-LiRPA} only against $\alpha$-CROWN and observed a similar trend, i.e., a loss of tightness above confidence $0.995$, further confirming the generality of our findings.

\paragraph{\textbf{Answers to Q2.}} For the second question, we consider the models trained on MNIST and CIFAR datasets, as provided in \citep{crown}. Hence, we evaluate the performance of our \texttt{PT-LiRPA}-based probabilistic verifier, alongside PROVEN \citep{proven}, Randomized Smoothing \citep{randomizedSmoothing}, and CROWN \citep{crown}, by testing for each model 10 random images from the corresponding test datasets. Specifically, we compare the maximum input perturbation $\varepsilon$ that can be certified for each method. CROWN \citep{crown} serves as the baseline for worst-case robustness certification, as employed in \citep{proven}. Based on the previous results, we set a sample size of $n = 350k,\;\xi=0.85$ and $1-2mp\geq0.99$ in our \texttt{PT-LiRPA}-based verifier. This confidence level is consistent with the settings used by \citet{proven} in their experiments.

\begin{table}[h!]
    \centering
    \caption{Comparison of \texttt{PT-LiRPA} with worst-case bound CROWN \cite{crown} and probabilistic approaches PROVEN \cite{proven}, Randomized Smoothing \cite{randomizedSmoothing} on different neural networks MNIST and CIFAR models. $\dagger$ results taken from the original paper \cite{proven} due to the code's unavailability to reproduce the results.}
    
    \resizebox{\textwidth}{!}{
    \begin{tabular}{l||cccccccc} 
    \hline 
    Certification method & Worst-case (CROWN) & PROVEN$^{\dagger}$~ & Rand. Smoothing & CROWN w/ \texttt{PT-LiRPA}~ & \texttt{PT-LiRPA} certification bound increase \\
    \multicolumn{1}{c||}{Confidence} & ~100\% & $\geq 99\%$ & $\geq 99\%$ & $\geq 99\%$ & w.r.t. CROWN, PROVEN, Rand. Smoothing \\ 
    \hline
    % MNIST 3$\times[20]$, ReLU & 0.02412 & 0.03828 & 0.016 & \textbf{0.04} & \textbf{1.66X} &  \textbf{1.04X} & \textbf{2.5X} \\ 
    % \hline
    MNIST 2$\times[1024]$, ReLU & 0.03566$\pm 0.011$ & 0.0556 & 0.0461$\pm$0.0106 & \textbf{0.0558$\pm$0.02068} & \textbf{1.6X}, \textbf{1.004X},  \textbf{1.21X} \\ 
    \hline
    MNIST 3$\times[1024]$, ReLU & 0.03112 $\pm$ 0.01076 & 0.03524 & 0.03452$\pm$0.0169 & \textbf{0.06652 $\pm$0.02501} & \textbf{2.14X}, \textbf{1.9X}, \textbf{1.93X}\\ 
    \hline
    MNIST 2$\times[1024]$, Tanh & 0.01827 $\pm$0.01331 & 0.02915 & 0.02301$\pm$0.0115 & \textbf{0.02949 $\pm$ 0.02515} & \textbf{1.61X}, \textbf{1.01X}, \textbf{1.28X} \\ 
    \hline
    MNIST 3$\times[1024]$, Tanh & 0.01244 $\pm$0.00468 & 0.01360 & 0.01294$\pm$0.0073 & \textbf{0.0257 $\pm$ 0.01205} & \textbf{2.07X}, \textbf{1.89X}, \textbf{1.99X} \\ 
    \hline
    MNIST 4$\times[1024]$, Sigmoid & 0.01975 $\pm$0.0111 & 0.02170 & 0.02439$\pm$0.019 &\textbf{0.05506$\pm$0.04035} & \textbf{2.79X}, \textbf{2.54X}, \textbf{2.26X} \\ 
    \hline
    CIFAR 5$\times[2048]$, ReLU & 0.002412 $\pm$ 0.00184 & 0.00264 & 0.00778$\pm$0.0104 &\textbf{0.00874 $\pm$ 0.00107} & \textbf{3.62X}, \textbf{3.31X}, \textbf{1.12X} \\ 
    \hline
    CIFAR 7$\times[1024]$, ReLU & 0.001984 $\pm$ 0.00089 & 0.00209 &  \textbf{0.006264$\pm$0.0061} &0.00471 $\pm$ 0.00273 & \textbf{2.37X}, \textbf{2.25X}, 0X \\
    \hline
    \end{tabular}}
    \label{tab:proven_comparison}
\end{table}

Tab.\ref{tab:proven_comparison} reports the results obtained. The \textit{``Worst-case"} column indicates the mean and standard deviation of the maximum $\varepsilon$ perturbation tolerated and provably certified using CROWN \citep{crown} for that model under consideration in the 10 random images tested, as in \citep{proven}. Instead, the PROVEN, Rand. Smoothing and \texttt{PT-LiRPA} columns report the certified mean $\varepsilon$ and standard deviation we achieve for the same images, using the three probabilistic approaches, sacrificing only a $10^{-2}$ of confidence. In general, we observe that the \texttt{PT-LiRPA}-based verifier can certify robustness levels up to 3.62 times higher than the worst-case baseline CROWN \citep{crown} and up to 3.31 and 2.26 times higher compared to PROVEN \citep{proven} and Rand. Smoothing \citep{randomizedSmoothing}, respectively. This demonstrates the significant advantage of our approach over other existing probabilistic methods. 

\begin{figure}[b]
    \centering
    \includegraphics[width=\textwidth]{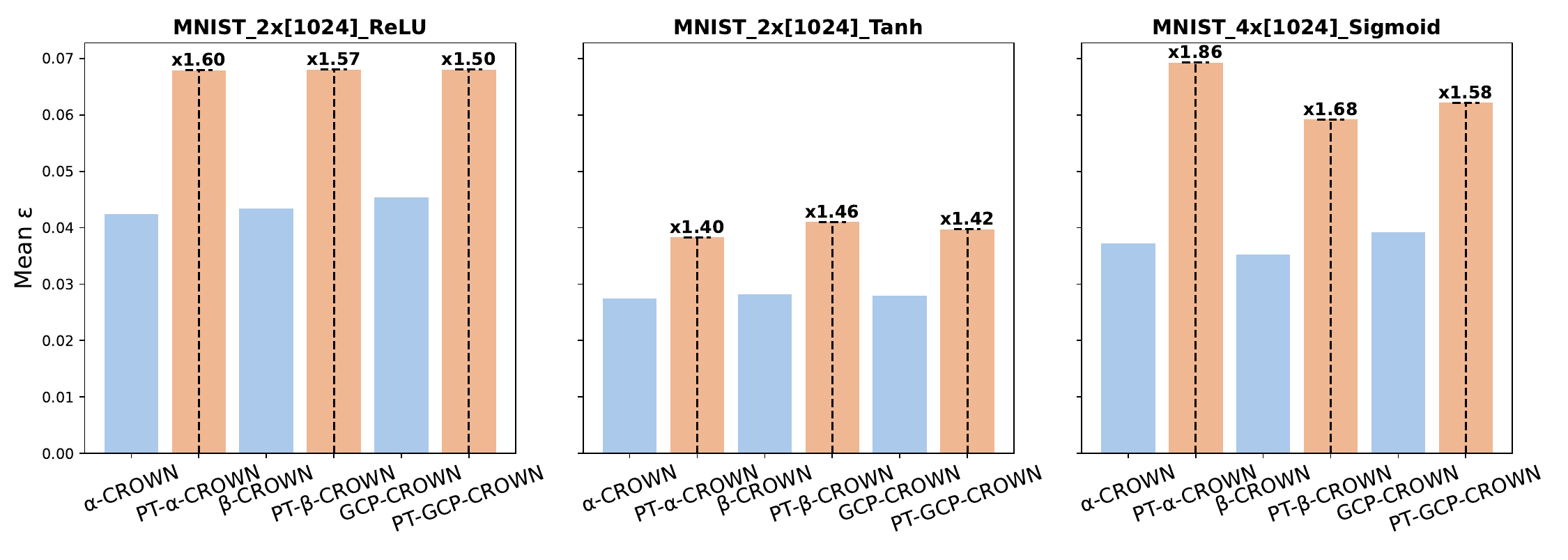}
    \vspace{-1cm}
    \caption{Comparison \texttt{PT-LiRPA} with worst-case bound CROWN \citep{crown}, $\alpha$-CROWN\citep{acrown}, $\beta$-CROWN \citep{bcrown}, GCP-CROWN\citep{GCP-CROWN} on MNIST\_2$\times[1024]$\_ReLU, MNIST\_2$\times[1024]$\_Tanh, and MNIST\_4$\times[1024]$\_Sigmoid models. On the x-axis, we report the original worst-case method and the corresponding probabilistic version using our \texttt{PT-LiRPA} framework. On the y-axis, we report, for each method, the mean maximum input perturbation $\varepsilon$ that can be certified on 10 random images.}
    \label{fig:appendix_further_empirical}
\end{figure}

Since the worst-case setting employed in this comparison is one of the first LiRPA-based approaches proposed in the literature, we conduct further analysis on the level of tightness we can achieve with respect to more recent LiRPA approaches such as $\alpha$-CROWN \citep{acrown}, $\beta$-CROWN \citep{bcrown}, and GCP-CROWN \citep{GCP-CROWN}. For each of these approaches, as well as their probabilistic counterparts based on our \texttt{PT-LiRPA} framework, we compute the mean input perturbation $\varepsilon$ that can be certified across 10 random images.
As shown in Fig.\;\ref{fig:appendix_further_empirical}, our probabilistic framework consistently certifies robustness levels at least 1.4 times higher than the worst-case baseline, even when using more recent LiRPA techniques. Notably, the results indicate that as the estimated reachable sets become more precise through over-approximation, the impact of our approach diminishes. Nonetheless, our framework can provide interesting safety information on the model's robustness level even with very tight provable reachable sets.

Finally, to assess the computational overhead introduced by the proposed approach, we perform an ablation study analyzing its impact on the total certification time, considering $n=350k$ samples and progressively larger network sizes. We highlight that the overall time complexity of the algorithm is polynomial in the network size, specifically $O\left(n \sum_{i=1}^{N} d_i d_{i-1}\right)$, where the term $\sum_{i=1}^{N} d_i d_{i-1}$ represents the complexity for a single sample, where $N$ is the total number of layers in the network, $d_{i-1}$ is the number of neurons in the preceding layer, and $d_i$ is the number of neurons in the current layer. The product $d_i d_{i-1}$ therefore indicates the number of multiplications required for the processing of layer $i$. Consequently, the total complexity is proportional to the number of samples multiplied by the cost of a single forward propagation through the entire network.
This aspect is well highlighted in the results of Tab. \ref{tab:proven_time_comparison}, where the time overhead of the sampling-based approach in \texttt{PT-LiRPA} is negligible (i.e., less than one second) in the overall certification time, even when employing a significantly large number of samples, thanks to the GPU acceleration employed in the certification process. Clearly, the total computation time of CROWN enhanced with \texttt{PT-LiRPA} is greater than that of using CROWN alone, as the probabilistic certification of a larger tolerable input perturbation results in a longer verification process.

\begin{table}[h!]
    \centering
    \caption{Time comparison between CROWN and CROWN enhanced with \texttt{PT-LiRPA} for the certification of the models in Table \ref{tab:proven_comparison}. The \textit{Total Cert. Time} column reports the overall time required to compute the average $\varepsilon$ perturbation that the model can tolerate across 10 random test images. The \textit{\# Computations of Interm. Bounds} column indicates how many times the intermediate bound computation procedure is invoked to determine the mean $\varepsilon$, while the \textit{\# Samples} column specifies the number of samples used in each instance of intermediate bound computation. Finally, the last two columns present the total overhead and the average time per call of the sampling-based approach used to compute the probabilistically optimal intermediate bounds.}
    \resizebox{\textwidth}{!}{
    \begin{tabular}{l||c||cccccc} 
    \textbf{Certification method} & \textbf{Worst-case (CROWN)} &  &  &\textbf{CROWN w/ \texttt{PT-LiRPA}}  &   &   \\
    \hline
     & \textbf{Total} & \textbf{Total} & \textbf{\# Computations}  &\textbf{\# Samples} & \textbf{Total interm. bounds} & \textbf{Mean Interm. Bounds}  \\ 
     & \textbf{Cert. Time} & \textbf{Cert. Time} & \textbf{of Interm. Bounds}  & & \textbf{Comp. time} & \textbf{Comp. time} \\
    \hline
    MNIST 2$\times[1024]$, ReLU & 35.7s & 47.4s & 253 & 350k & 0.27s &  0.001s \\ 
    \hline
    MNIST 3$\times[1024]$, ReLU & 37.73s & 54.85s & 248 &  350k & 0.34s  & 0.0014s\\ 
    \hline
    MNIST 2$\times[1024]$, Tanh & 23.02s & 36.73s  & 165 & 350k & 0.22s & 0.0013s \\ 
    \hline
    MNIST 3$\times[1024]$, Tanh &  31.47s &  59.82s  & 186 & 350k & 0.26s & 0.0014s  \\ 
    \hline
    MNIST 4$\times[1024]$, Sigmoid & 40.26s & 64.61s & 209 & 350k & 0.33s & 0.0016s\\ 
    \hline
    CIFAR 5$\times[2048]$, ReLU & 26.3s & 82.42s & 120 & 350k & 0.48s & 0.004s \\ 
    \hline
    CIFAR 7$\times[1024]$, ReLU & 21.02s & 73.9s & 154  & 350k & 0.37s & 0.0024s \\
    \hline
    \end{tabular}}
    
    \label{tab:proven_time_comparison}
\end{table}

Importantly, the time comparison is conducted only against CROWN as the goal of this experiment is to show that the additional cost introduced by \texttt{PT-LiRPA} is negligible in the overall verification time, while still producing tighter output bounds. Importantly, once the LiRPA baseline is fixed (e.g., $\alpha$-CROWN, $\beta$-CROWN, GCP-CROWN), the subsequent linearization procedure is identical whether using the original method or our \texttt{PT-LiRPA} variant. The only difference lies in the way intermediate bounds are computed, which are then used to construct the diagonal matrices and bias terms. Consequently, we argue that measuring the overhead against CROWN is representative, since the additional cost introduced by \texttt{PT-LiRPA} does not depend on which LiRPA baseline is employed.

\paragraph{\textbf{Answers to Q3.}} To answer the last question, we integrate our \texttt{PT-LiRPA} in the $\alpha,\beta$-CROWN toolbox (\url{https://github.com/Verified-Intelligence/alpha-beta-CROWN}) and perform a final experiment on different benchmarks of the VNN-COMP 2022 and 2023 \citep{VNN-comp2022,VNN-comp2023}. We point out that the comparison between probabilistic and provable verifiers is performed to have a ground truth (when possible), and to highlight the valuable help that probabilistic approaches can have in solving challenging instances to be verified.\\
To keep the paper self-contained, we report below a brief overview of the selected benchmarks.

\begin{itemize}
    \item \textit{ACAS xu} \citep{ACAS,Reluplex} benchmark 2023: includes ten properties evaluated across 45 neural networks designed to provide turn advisories for aircraft to prevent collisions. Each neural network consists of 300 neurons distributed over six layers, using ReLU activation functions. The networks take five inputs representing the aircraft's state and produce five outputs, with the advisory determined by the minimum output value. Here, we verified only property 3, which returns unsafe if the clear of conflict (COC) output is minimal, with a max computation time of 116s.
    \item \textit{TllVerifyBench} benchmark 2023: this benchmark features Two-Level Lattice (TLL) neural networks with two inputs and one single output. These models are then transformed into MLP ReLU networks where the output properties consist of a randomly generated real number and a randomly generated inequality direction to be verified. Here we verify all 32 instances of the VNN-COMP 2023 with a timeout of 600s for each property.
    \item \textit{CIFAR\_biasfield} benchmark 2022: this benchmark focuses on verifying a Cifar-10 network under bias field perturbations. These perturbations are modeled by creating augmented networks that reduce the input space to just 16 parameters. For each image to be verified, a distinct bias field transform network is generated, consisting of a fully connected transform layer followed by the Cifar CNN with 8 convolutional layers with ReLU activations. Each bias field transform network has 363k parameters and 45k nodes. Here, we test all 72 properties with a timeout set to 300s for each one.
    \item \textit{TinyImageNet} benchmark 2022: consists of CIFAR100 image classification ($56\times56\times3$) with Residual Neural Networks (ResNet). Here, we consider the medium network size composed of 8 residual blocks, 17 convolutional layers, and 2 linear layers. For \textit{TinyImageNet-ResNet-medium}, we verify all 24 properties with a timeout of 200 seconds for each property.
\end{itemize}

In general, we selected benchmarks where the state-of-the-art $\alpha,\beta$-CROWN method is unable to solve some of the instances within the time constraints. This set of experiments aims to confirm our hypothesis regarding the effectiveness of having tighter estimated reachable sets for verification purposes. In detail, our intuition is that even though our procedure requires a little initial overhead, with tighter estimated reachable sets, we can achieve more precise final output bounds, potentially reducing the cases where the verification approach can not make a decision and must resort to a split in the BaB process, thus resulting in more efficient overall verification time.\\
Before starting to verify these instances,  we explore the effect of varying incremental sample sizes for a fixed $\xi=0.85$ on the computation of estimated reachable sets with neural networks of significant size. 

\begin{figure}[h!]
    \centering
    \includegraphics[width=0.7\linewidth]{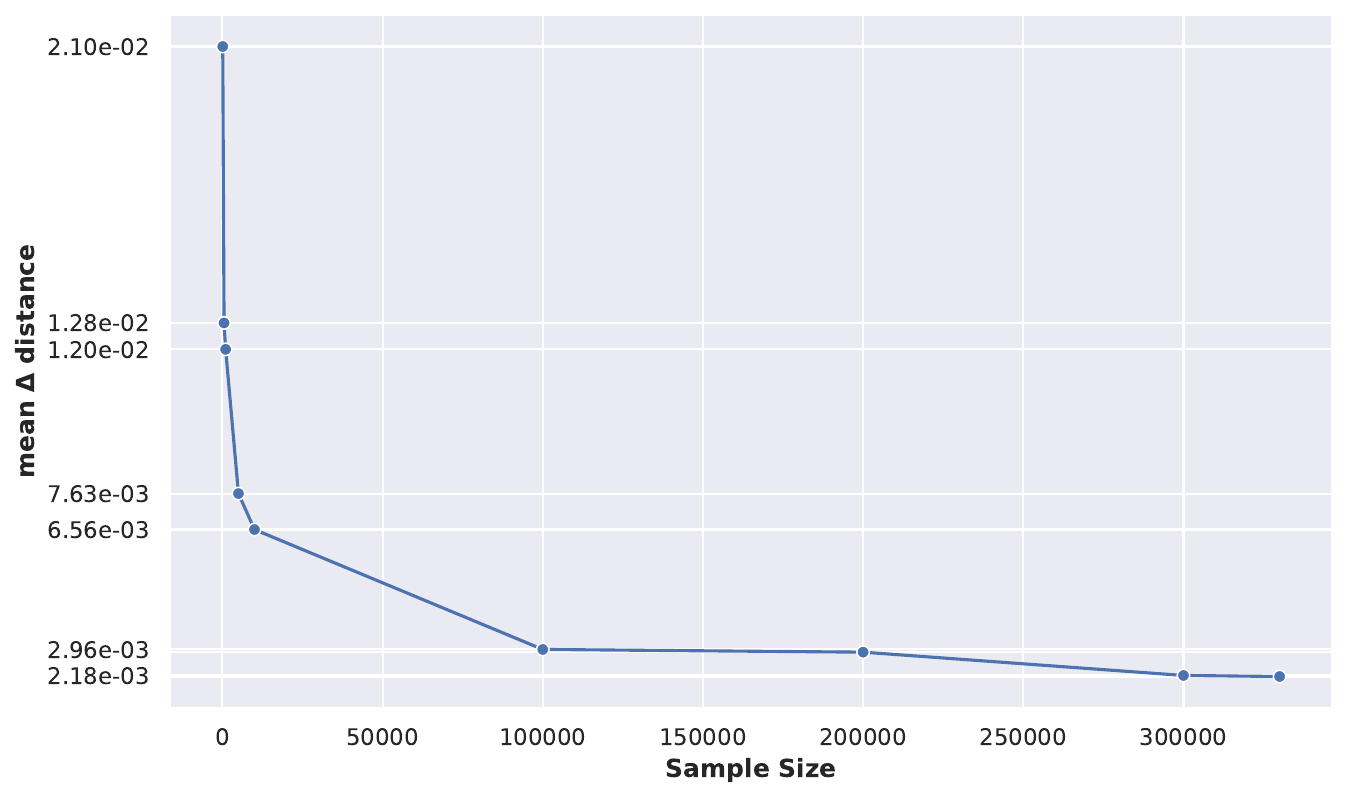}
    \caption{Estimated reachable sets at convergence for the increasing sample size in \textit{CIFAR\_biasfield} benchmark. $y$-axis reports the mean distance between  estimated reachable sets using $350k$ samples (as reference) and the one using $[100, 500, 1k, 5k, 10k, 100k, 200k, 300k, 330k]$, respectively.}
    \label{fig:convergence}
\end{figure}

In detail, we focus on the \textit{CIFAR\_biasfield} benchmark and set a confidence level of $1-2mp \geq 99\%$. In detail, we compute the mean distance between estimated reachable sets using $350k$ samples as reference and the one using progressively increasing the sample size until the difference between successive estimated reachable sets exceeds the threshold of $\Delta=0.001$. Specifically, we begin with $100$ samples and progressively increase the sample size until the difference between successive estimated reachable sets exceeds the threshold. Our results, detailed in Fig.\;\ref{fig:convergence}, indicate that stable estimated reachable sets, in this scenario, can be obtained with sample sizes ranging from $300k$ to $350k$ as the mean distance between estimated bounds is strictly less than $\Delta = 0.001$. Hence, employing a sample size of $350k$ samples results in a valid choice even for considerably large networks. We recall once again that propagating a large number of samples, such as $350k$, requires a negligible computational effort and time (as shown in the results of Tab. \ref{tab:proven_time_comparison}) due to batch processing and GPU acceleration. The principal limitation arises from GPU memory capacity, since larger sample sizes may increase the risk of memory errors relative to CPU-based propagation.

\begin{table}[t]
%\scriptsize
\centering
\caption{Results on VNN-COMP 2022-2023 benchmarks. Results marked in \textbf{bold} report improved performance in terms of verified accuracy (\% sat instances/all instances) and total verification time for the specific benchmark tested, w.r.t. a worst-case verification approach.}

\resizebox{\textwidth}{!}{%
\begin{tabular}{c||cccccccc}
\hline
\textbf{Benchmark} & \textbf{Method} & \textbf{Confidence} & \textbf{Verified accuracy} & \textbf{\#safe (unsat)} & \textbf{\#unsafe (sat)} & \textbf{\#unkwown} & \textbf{Tot verification time} \\
\hline
\textit{ACASxu} & $\alpha$,$\beta$-CROWN & $100\%$ & 93.33\% & 42 & 3 & 0 & 26s\\
 & $\alpha$,$\beta$-CROWN &&&&&\\
 & w/\texttt{PT-LiRPA} & $\geq 99\%$ & 93.33\% & 42 & 3 & 0 & \textbf{16.37s}\\
 \hline
\textit{tllVerifyBench} & $\alpha$,$\beta$-CROWN & $100\%$ & 46.875\% & 15 & 17 & 0 & 90.2s\\
 & $\alpha$,$\beta$-CROWN &&&&&\\
 & w/\texttt{PT-LiRPA} & $\geq 99\%$ & 46.875\% & 15 & 17 & 0 & 92s\\
 \hline
\textit{CIFAR\_biasfield } & $\alpha$,$\beta$-CROWN & $100\%$ & 95.83\% & 69 & 1 & 2 & 1553.5s\\ & $\alpha$,$\beta$-CROWN &&&&&\\
 & w/\texttt{PT-LiRPA} & $\geq 99\%$ & \textbf{98.61\%} & \textbf{71} & 1 & \textbf{0} & \textbf{408.7s}\\
 \hline
\textit{CIFAR\_tinyimagenet} & $\alpha$,$\beta$-CROWN & $100\%$ &  62.5\% & 15 & 3 & 6 & 1429.6s\\ & $\alpha$,$\beta$-CROWN &&&&&\\
 & w/\texttt{PT-LiRPA} & $\geq 99\%$ & \textbf{87.5\%} & \textbf{21} & 3 & \textbf{0} & \textbf{425.6s}\\
 \hline
\end{tabular}}

\label{tab:results_VNNCOMP}

\end{table}
In Tab.\;\ref{tab:results_VNNCOMP} we report our results on the VNN-COMP, where we consider an increased difficulty for the verification process. We start with the simpler benchmark \textit{ACASxu} \citep{ACAS,Reluplex}, and we test property 3. This property is particularly interesting as it holds for 42 of the 45 models tested, thus allowing us to verify the improvement in terms of time and verification accuracy. In the first row of Tab.\;\ref{tab:results_VNNCOMP}, we can notice that by sacrificing only a $0.01\%$ of confidence, $\alpha,\beta$-CROWN enhanced with \texttt{PT-LiRPA} achieves the same verified accuracy in less verification time, thus confirming our intuition. 
Interestingly, we observe that tighter bounds are not always beneficial in general. Specifically, in cases where a PGD attack succeeds despite loose bounds, using tighter bounds does not lead to further improvements. Additionally, in some scenarios, less accurate bounds from vanilla LiRPA methods could be quickly refined by BaB, still resulting in efficient verification time. This is exemplified by the \textit{tllVerifyBench} experiments, where even sacrificing a $0.01\%$ of confidence, \texttt{PT-LiRPA} produced tight estimated reachable sets but achieved the same verified accuracy with a minor overhead in bounds computation.\\
Crucially, the real benefit of our \texttt{PT-LiRPA} arises on more challenging verification benchmarks such as \textit{CIFAR\_biasfield} and \textit{CIFAR\_tinyimagenet}. Both these benchmarks are image-based verification tasks and thus allow us to show the scalability and the impact of tightened estimated reachable sets on large networks using the proposed approach. Crucially, in these two last benchmarks, sacrificing a $0.01\%$ of confidence, we obtain significant improvements in verification results with respect to worst-case $\alpha,\beta$-CROWN. In detail, in both \textit{CIFAR\_biasfield} and \textit{CIFAR\_tinyimagenet}, we achieved higher verified accuracy without incurring any \textit{unknown} answer and with significantly less verification time. These final results demonstrate the effectiveness and the potential impact of using \texttt{PT-LiRPA} for verification purposes, showing the advantage of incorporating estimated reachable sets in handling challenging instances that are difficult to solve with provable solvers.

\section{Assumptions and Limitations}

The proposed \texttt{PT-LiRPA} framework builds on several assumptions that define its applicability and theoretical guarantees. Our probabilistic framework relies on uniform random sampling within the perturbation region $\mathcal{C}_{\mathbf{x}_0, \epsilon}$ to estimate probabilistically tight reachable sets. Consequently, our theoretical probabilistic guarantees, derived from \citet{wilks}'s tolerance limit theorem and extended using extreme value theory \citep{haan2006extreme}, hold under the assumption that the samples are independent and representative of the true input distribution within $\mathcal{C}$.\\
Importantly, by accepting robustness certificates that hold for a fraction $R$ of the perturbation region, the theoretical and practical tool derived from \citet{wilks}'s results, i.e., Theorem \ref{theorem:PT_lirpa_wilks}, remains a valuable solution, as it provides a closed-form expression to compute the number of samples required for any desired confidence level $ \psi$ and coverage ratio $R$. To address this limitation and extend the analysis to probabilistic certificates valid over the entire perturbation region, thus aligning with the probabilistic verification literature, we further based our theoretical derivation on extreme value theory and the results of \citet{de1981estimation}. In this case, our derivation in Theorem \ref{theorem:pt_lirpa_guarantees} provides, for any choice of sample size $n$ and confidence error $p$, a guarantee that holds for all input values $\vx \in \mathcal{C}$. However, unlike Wilks’ approach, EVT does not yield a closed-form expression to determine the minimum number of samples required to achieve a desired precision and confidence level. Consequently, the balance between probabilistic soundness and the tightness of the resulting output bounds becomes more empirical, as confirmed in our experiments. In fact, the tightness of the computed bounds and the computational efficiency of the method depend on the chosen sample size $n$, and the estimate of the tail distribution $\nu = \lfloor n^\xi \rfloor$, which must balance verification accuracy and computational cost. Despite these limitations, we argue that \texttt{PT-LiRPA} complements deterministic verification methods by offering practical, quantifiable robustness guarantees for challenging instances where worst-case formal verification is either overly conservative or computationally infeasible.
\section{Conclusion}

We introduced \texttt{PT-LiRPA}, a novel probabilistic framework that combines LiRPA-based formal verification of deep neural networks approaches with a sampling-based technique. We provide a rigorous theoretical derivation of the correctness of our approach, complementing, for the first time, statistical results on the tolerance limit, with qualitative bounds on the error magnitude of a sampling-based approach employed to estimate reachable set domains. Our approach provides tighter linear bounds, significantly improving both the accuracy and verification efficiency while maintaining provable probabilistic guarantees on the soundness of the verification result. Empirical results demonstrate that \texttt{PT-LiRPA} outperforms related probabilistic methods, particularly in robustness certification, decreasing the confidence in the result by infinitesimal amounts. Additionally, we show the potential of our probabilistic approach for verifying challenging instances where the formal approaches fail.

\section*{Acknowledgments}
This work has been supported by PNRR MUR project PE0000013-FAIR. 

%Bibliography
\bibliographystyle{plainnat} 
\bibliography{references}

\end{document}